\documentclass[11pt]{article}

\usepackage{hyperref}

\usepackage{fixltx2e,amsmath,amsfonts,amsthm} \usepackage{titlesec}
	\usepackage{graphicx}
	\usepackage{subeqnarray}
	\usepackage[mathscr]{euscript}
	\usepackage{enumerate}
	\usepackage{longtable}
	\usepackage[T1]{fontenc}
	\usepackage[toc,page,title]{appendix}
    \usepackage{import}	
    \usepackage[table]{xcolor}
	\usepackage{array,ragged2e}
	\usepackage{arydshln,leftidx}

	\MakeRobust{\eqref}
	
	\newcommand*{\myfont}{\fontfamily{phv}\selectfont}
	\DeclareTextFontCommand{\ggplotFont}{\myfont}

	\usepackage{caption}
	\usepackage{multirow}
	\usepackage{dsfont}
	\usepackage{amssymb}
	\usepackage{float}
	\usepackage{multirow}
	\usepackage[shortlabels]{enumitem}
	
	\usepackage{mathtools}

\usepackage{bm}
\theoremstyle{definition}
\usepackage{enumitem}

\newcommand{\E}{\mathbb{E}}
\newcommand{\cD}{\mathcal{D}}
\newcommand{\cI}{\mathcal{I}}

\theoremstyle{remark}
\newtheorem{remark}{Remark}

\DeclareMathOperator*{\sign}{sign}

\DeclareMathOperator*{\argmax}{arg\,max}
\DeclareMathOperator*{\argsup}{arg\,sup}
\DeclareMathOperator*{\median}{median}
	
\newtheorem{theorem}{Theorem}[section]
\newtheorem{lemma}[theorem]{Lemma}

\usepackage[lined]{algorithm2e}
\SetAlCapSkip{1ex}
\SetNlSty{}{}{} 
\allowdisplaybreaks

\usepackage{natbib}
\setcitestyle{authoryear,open={(},close={)}}

\begin{document}
	

	
\title{Probabilistic Bisection with Spatial Metamodels}
	
\author{Sergio Rodriguez
and Michael Ludkovski
}
%
%
%
\date{Department of Statistics and Applied Probability, \\ University of California, Santa Barbara, Santa Barbara, CA, USA 93106 \\ \url{{srodriguez,ludkovski}@pstat.ucsb.edu} }

	

\maketitle

\begin{abstract}
	Probabilistic Bisection Algorithm performs root finding based on knowledge acquired from noisy oracle responses. We consider the generalized PBA setting (G-PBA) where the statistical distribution of the oracle is \textit{unknown} and location-dependent, so that model inference and Bayesian knowledge updating must be performed simultaneously. To this end, we propose to leverage the spatial structure of a typical oracle by constructing a statistical surrogate for the underlying logistic regression step. We investigate several non-parametric surrogates, including Binomial Gaussian Processes (B-GP), Polynomial, Kernel, and Spline Logistic Regression. In parallel, we develop sampling policies that adaptively balance learning the oracle distribution and learning the root. One of our proposals mimics active learning with B-GPs and provides a novel look-ahead predictive variance formula.
The resulting gains of our Spatial PBA algorithm relative to earlier G-PBA models are illustrated with synthetic examples and a challenging stochastic root finding problem from Bermudan option pricing. 	
\end{abstract}
	
	

	\section{Introduction}
	\label{sec:PBA}
	
The Probabilistic Bisection Algorithm (PBA) is a numeric estimation procedure for learning an unknown parameter $x^{*}$ (defined on a bounded search space, without loss of generality $[0,1]$) based on the information provided by noisy responses observed independently at sampling/querying sites $x_{1:n}:=(x_{1},\ldots,x_{n})$. In the context of the Stochastic Root Finding Problem (SRFP)~\citep{pasupathy2011stochastic,waeber2011bayesian}, the PBA can be used to learn the root, $x^{*}:=h^{-1}(0)$, of a noisily observed real-valued function $h:[0,1]\rightarrow \mathbb{R}$. Specifically we consider an oracle of the form
\begin{equation}
\label{eq:pba-oracle}
Y(x_{n}) := \sign Z(x_{n}),
\end{equation}
where the structural form of the (random) responses $Z(x_{n})$ in~\eqref{eq:pba-oracle} is given by
\begin{equation}
\label{eq:z-oracle}
Z(x_{n}):= h(x_{n}) + \epsilon(x_{n}).
\end{equation}
The noise component $\epsilon(x_{n})$ in~\eqref{eq:z-oracle} is assumed to be a \textit{symmetric} heteroscedastic (i.e.,~input-dependent) random term with mean $\mathbb{E}[\epsilon(x_{n})] = 0$ and variance $\mathbb{V}ar(\epsilon(x_{n})) := \sigma^{2}(x)$, with independent realizations across different oracle calls.

 The PBA leverages the classical \textit{bisection search} strategy in a noise-free setting: repeatedly halve the search region and then select a subinterval in which a root must lie for further processing. The stochastic PBA accounts for the noise in the oracle responses by considering $x^{*}$ as the realization of an absolutely continuous random variable $X^{*}\sim g_{0}$ with prior density $g_{0}$ supported on $[0,1]$. The PBA then works with the {sign} of the noisy function evaluations~\eqref{eq:pba-oracle}, which provide information as to whether $x^*$ lies to the left or to the right of a given $x_{n}$, in order to subsequently update a posterior density for $X^{*}$,
\begin{equation}
\label{eq:g_n}
{g_{n}(X^{*}):=p(X^{*}|Y_{1:n},x_{1:n}).}
\end{equation}
Thus, $g_n$ is the pdf of the root location $X^{*}$ conditional on the history $Y_{1:n}:=(Y_{1}(x_{1}),\ldots,Y_{n}(x_{n}))$ of oracle responses, the sampling locations $x_{1:n}$ and the prior $g_{0}$. The posterior~\eqref{eq:g_n} then serves for the twin purposes of guiding the election of the next sampling location $x_{n+1}$ at which to query~\eqref{eq:pba-oracle}, as well as to provide a point estimator $\hat{x}_{n}$ for $X^{*}$ (e.g., the posterior median or mean of $g_{n}(\cdot)$).

Due to the noise term $\epsilon(x_{n})$ in the simulation outputs $Z(x_{n})$ in \eqref{eq:pba-oracle}, the responses $Y(x_{n}) =\sign Z(x_{n})$ translate into potentially inaccurate oracle directions. To account for such ``mistakes'',  the PBA considers the probability of \textit{correct sign},
\begin{equation}
\label{eq:p-correct1}
p(x_{n}) := \mathbb{P} \bigl(Y(x_{n}) = \sign\{x_{n}-x^{*}\}\bigr), 
\end{equation}
henceforth referred to as oracle {specificity} or \textit{accuracy}, which is then used to update knowledge about $X^{*}$ by re-weighting the current $g_{n}$ proportionally to $p(x_{n})$. \cite{waeber2013bisection} provided an explicit recursive updating formula under the restrictive condition that the oracle accuracy is a known constant $p(x)=p^{*}>1/2$ for all $x\in (0,1)$. This assumption of {spatial oracle stationarity} would tend to be met in applications where the transition between regions in $h$ is abrupt. As an example, if a city's water supply were contaminated with a dangerous chemical we would want to localize the extent of contamination as quickly as possible, and if the chemical did not dissolve well in water but instead tended to stay concentrated, we would face a situation with such abrupt transition between contaminated and uncontaminated water~\citep{powell2012optimal}.

However, in the more general and practical case, including the SRFP in~\eqref{eq:z-oracle}, $p(x)$ is unknown and location-dependent and hence must be itself \textit{estimated}. The Generalized PBA (G-PBA) that we developed in~\cite{rodriguez2017generalized} extends the classical PBA by using the observed data to construct a point estimate, $\hat{p}(x)$, for $p(x)$, as well as to learn the root location $X^{*}$ in parallel. The proposed estimators $\hat{p}(x_{n+1})$ under the aforementioned G-PBA paradigm were  constructed \textit{locally} at $x_{n+1}$ (i.e., without using information from previous locations $x_{1:n}$). As such, they were robust to arbitrary specification of $p(\cdot)$ and made minimal assumptions about the oracle.

\textbf{Surrogate modeling. }In this article we construct a \textit{spatial} G-PBA by modeling the entire oracle accuracy $x \mapsto p(x)$ using a \emph{surrogate}.

The surrogate relies on two main premises:
(i) Due to
	symmetrical noise distribution of the functional responses~\eqref{eq:z-oracle}, the oracle accuracy~\eqref{eq:p-correct1} can be re-formulated as $p(x){=}\max\{\theta(x),1-\theta(x)\}$, where
	\begin{equation}
	\label{eq:ProbPositiveResponse}
	\theta(x) := \mathbb{E} \left[1_{\{Z(x)>0\}} \right]
	\end{equation}
	is the probability of observing a \textit{positive} oracle response. Thus, inference on $p(x)$ can be performed by inferring $\theta(x)$ first and then plugging in a spatial-based estimate $\hat{\theta}_{n}(x)$ into $p(x){=}\max\{\hat{\theta}_{n}(x),1-\hat{\theta}_{n}(x)\}$, and (ii) the smoothness of the map $x\mapsto \theta(x)$, implies that $p(x)$ and $p(x')$ should be similar when $x$ and $x'$ are deemed close to each other.

The spatial structure is natural in the root-finding context and provides two key benefits. On the one hand, it improves estimation of a given $p(x_n)$ through leveraging the knowledge acquired at previous sampling locations $x_{1:n-1}$. On the other hand, it enables better sampling strategies by furnishing a prediction $\hat{p}(x)$ at arbitrary, unsampled sites $x$. In contrast, in G-PBA, $\hat{p}(x_{n+1})$ was only available \emph{a posteriori} after sampling at $x_{n+1}$.

The resulting Spatial G-PBA strategy blends the root-centric framework of PBA and the function-centric paradigm of response surface modeling (RSM). Indeed, a further alternative for solving the SRFP would be to learn the entire $\theta(\cdot)$ and then take $\hat{x} = \hat{\theta}^{-1}(0.5)$ since $h(x^*) = 0 \Leftrightarrow \theta(x^*) = 0.5$. Thus, stochastic root-finding can be recast as a (localized) learning task, namely contour-finding for $\theta(\cdot)$ at the level $0.5$. Strategies similar to Bayesian optimization \citep{JonesSchonlauWelch98,ChevalierPicheny13} can then be employed to efficiently target this objective during sequential design.   Nevertheless, several challenges are encountered with such an approach that are circumvented in PBA. First, a major feature of PBA is full uncertainty quantification: the algorithm provides not only the point estimate $\hat{x}$ but also the entire posterior distribution $f_n$ of $X^*$ conditional on the data. Typical RSM models return only point estimates (or pointwise credible intervals) of $\theta(x)$; the latter are difficult to ``invert'' into uncertainty about $\theta^{-1}(0.5)$~\citep{AzzimontiBect16}. Second, existing experimental design approaches for contour-finding are developed only for simple models (e.g.~with zero or constant observation noise), and their performance in a complex stochastic setting like ours is poorly understood. In contrast, the PBA framework explicitly targets the goal of reducing uncertainty about $X^*$. PBA moreover exploits the structural knowledge of a \emph{unique} root to speed up estimation, an option that is not available in contour-finding. Third, contour-finding usually assumes continuous response, and nontrivial modifications (essentially ``logistic'' contour-finding) are necessary to handle binary $Y_n \in \{-1, 1\}$. In contrast, PBA is intrinsically  designed for binomial responses.

Given the above discussion, we construct a hybrid algorithm that borrows the best of both worlds. We exploit the smoothness of $h$ that implies spatial dependence in $\theta(\cdot)$ and hence accelerates learning the oracle. At the same time, we employ the paradigm of PBA to construct the knowledge state $f_n$ (a pseudo-posterior of $X^*$) that is the primary driver of sampling decisions and uncertainty quantification. For the RSM component, we rely on two key concepts. First, we investigate non-parametric architectures that have the flexibility to consistently learn the entire response $x \mapsto  \theta(x)$ and to handle non-uniform simulation designs. The latter point is key as we wish to organically refine the surrogate in regions where more inputs are placed (namely close to the root), but at the same time give a good global fit. To handle the binary responses \eqref{eq:pba-oracle} we employ \textit{logistic regression} which represents the probability of observing a positive response $\theta(x)= \mathbb{E}[1_{Z_{n}(x)>0}]$ via a \textit{latent} process $\varphi(x) := \mbox{logit}(\theta(x))$. Other link functions can also be used but as we show in the sections below, the canonical Bernoulli logit link is best suited for our needs.
For capturing the spatial surrogate $\varphi$, we consider Gaussian Process (GP) models, as well as spline, kernel and polynomial logistic regressions. Second, we apply \textit{batched sampling} that significantly lowers the computational overhead of surrogate construction and improves the learning of $\theta(\cdot)$. Replicated experimental designs allow to blend the local inference of $\theta(x_n)$ with the global fitting of the surrogate. They also offer a new aspect of sequential design, namely adaptive replication, linking to the active learning literature in Bayesian optimization.

\textbf{Summary of Contributions and Related Literature. }Our contributions can be traced along two directions. First, the developed G-PBA algorithm extends existing probabilistic bisection schemes in \cite{jedynak2012twenty,waeber2011bayesian,waeber2013bisection,waeber2013probabilistic,frazier2016probabilistic}, in particular making them much more efficient even when the oracle distribution is a priori unknown. Thus, we contribute to the stochastic root-finding toolkit. Second, our work has independent interest in terms of applications of binomial GP  (B-GP) surrogates. To this end, we provide an original result for the look-ahead variance formula of a binomial GP, that to our knowledge is not available in existing literature. This provides a new application of B-GPs in the context of active learning, linking to related work in \cite{kapoor2007active,tesch2013expensive,wang2016optimization}.

In the extensive numerical section, we demonstrate that by introducing a spatial surrogate we are able to improve PBA's accuracy in the root estimation. By using three different synthetic examples, we show that absolute residuals decay faster using spatial surrogates than their corresponding (local) G-PBA methods. Additionally, the posterior uncertainty of the root estimate, as measured by the width of the posterior credible interval (CI), is reduced {and, most importantly, the probability coverage (i.e., the proportion of macro-runs where the CI contains the actual root value) drastically increases with respect to their local counterparts}.

The rest of the paper is organized as follows. In Section~\ref{sec:SpatialEstimation} we describe the model methodology used to provide a spatial estimate for $p(\cdot)$. Section~\ref{sec:AdaptiveSampling} then describes an adaptive batching/replication scheme in order to determine the number of replicates $a_{n+1}$ given an estimated surrogate model. Section~\ref{sec:sampling} develops the surrogate-based sampling schemes for the SRFP. In Sections~\ref{sec:NumericExamples} and~\ref{sec:Results-AmericanOption} we illustrate the developed Spatial G-PBA with several synthetic examples and a challenging real-world application coming from an Optimal Stopping problem.

\section{Spatial Modeling of the Oracle}\label{sec:SpatialEstimation}
PBA works in the sequential setting, adaptively picking query sites $x_{n+1}$ given information from previous queries. The latter is summarized via a knowledge state $f_n$ which captures the Bayesian formulation of the SRFP, translating the task of learning the root
$X^{*}$ into quantifying the corresponding posterior uncertainty.  At each iteration $n=1,\ldots$, the oracle is called $a_{n}\geq 1$ times at a fixed sampling location $x_{n}$  with the responses {$(Z_{1}(x_{n}),\ldots,Z_{a_{n}}(x_{n}))$} aggregated  via the total number of positive signs observed at $x_{n}$:
\begin{equation}
\label{eq:PositiveResponses-Bn}
B_{n}(x_{n}) := \sum_{j=1}^{a_{n}}1_{\{Z_{j}(x_{n})>0\}}.
\end{equation}
The overall information set by round $n$ is therefore $\cD_n := (B_{1:n},a_{1:n},x_{1:n})$. We shall distinguish between the macro counter $n$ that keeps track of PBA iterations, i.e.~the number of distinct sites $x_{1:n}$, and the wall clock $T_{n}:=\sum_{i=1}^{n}a_{i}$ that counts total number of function evaluation and hence the overall computational expense. Occasionally, we abuse the notation, switching between writing $f_n$ and $f_{T_n}$.

Given the current knowledge state $f_{n}$ and a total simulation budget of $T>0$ wall-clock iterations,  the fundamental G-PBA loop is:
\begin{algorithm}
	\SetKwInOut{Input}{input}\SetKwInOut{Output}{output}
	Initialize $T_{0}$ and $\hat{p}_0$\;
	\While{$T_{n}<T$}{
  	Choose $x_{n+1}$ based on $f_{n}$ and $\hat{p}_n(\cdot)$\; {Call the oracle $a_{n+1}$ times at $x_{n+1}$ and record $B_{n+1}$ as in \eqref{eq:PositiveResponses-Bn}}\;
  	 Use $(a_{n+1}, B_{n+1})$ to update $f_{n+1} \leftarrow \Psi( f_{n}, B_{n+1}; \hat{p}(x_{n+1}),a_{n+1})$ and re-fit $\hat{p}_{n+1}(\cdot)$\;
	Update wall-clock time: $T_{n+1}\leftarrow T_{n} + a_{n+1}$ and increment $n \leftarrow n+1$ \;
	}
	\nllabel{alg:updating-BasicGPALoop}
	\Return Knowledge state  $f_{n}\simeq g_{n}$ and estimator for the root location $\hat{x}_{n}$.
	\caption{G-PBA iterations}\label{alg:BasicGPBALoop}
\end{algorithm}

In the general case of unknown and varying oracle specificity, the key ingredients of Algorithm~\ref{alg:BasicGPBALoop} are:

\begin{enumerate}[label=(GPBA-\Roman*),align=left]
	\item \label{prop:pba1} statistical learning sub-routine for $\hat{p}(\cdot)$.
	\item \label{prop:pba2} the mechanism to update knowledge states $\Psi: f_{n} \rightarrow f_{n+1}$;
	\item \label{prop:pba3} the sampling rule $\eta$ for selecting $x_{n+1} =\eta(f_{n}; \hat{p}_n)$ given $f_{n}$ and $\hat{p}_n$.
\end{enumerate}

\textbf{Learning sub-routine for $\hat{p}(\cdot)$. }For estimating the oracle specificity, G-PBA relies on the aggregated number $B_n$ of positive signs observed at $x_{n}$ across $a_n$ oracle queries.  Replicates decouple the problems of learning $X^{*}$ and of learning $p(\cdot)$; they also boost the signal-to-noise ratio which allows \textit{faster} convergence at the macro-level. The original G-PBA did this locally, returning an  estimate $\hat{p}(x_{n})$ depending solely on $(x_n, a_n, B_n)$. In this paper we extend~\ref{prop:pba1} by introducing a {surrogate} model $x\mapsto \varphi(x)$ on~\eqref{eq:ProbPositiveResponse} which is built upon the \emph{history} of binomial responses $B_{1:n} := (B_{1}(x_{1}),\ldots,B_{n}(x_{n}))$ observed at all queried $x_{1:n}$. In particular, we have that $B_{n}(x) \sim \mathsf{Bin}(a_{n},\theta(x))$ is a binomial random variable which is statistically \textit{sufficient} and unbiased for $\theta(x)$. To learn $\theta(\cdot)$ we therefore regress $B_{1:n}$ against the locations $x_{1:n}$, linking each $x_{i}$ to $\theta(x_{i})$ via the canonical Bernoulli link function:
	\begin{equation}
	\label{eq:SurrogateModel0}
	\log\left(\frac{\theta(x_{i})}{1-\theta(x_{i})}\right) = \varphi(x_{i}),\quad i=1,\ldots,n.
	\end{equation}
We consider two families for $\varphi(\cdot)$: (A)  Gaussian random field approach \citep{rasmussen2006gaussian} that takes $\varphi$ as a latent Gaussian process (GP) and outputs the posterior distribution $p(\varphi_*|\cD_n)$; (B) a linear additive model that assumes that $\varphi$  is an element of a linear space $\mathcal{H}$ spanned by a collection of basis functions, i.e.,  $\varphi(x) = \sum_{j=1}^{p}\beta_{j}\phi_{j}(x)$, with the coefficients $\bm{\beta} := (\beta_{1},\ldots,\beta_{p})$ fitted, for example, by penalized MLE.

Given the fitted surrogate $\hat{\varphi}_{n}$, the estimate for $p(\cdot)$ is a plug-in estimate {of the form}:
\begin{equation}
\label{eq:Estimated-pn}
\hat{p}_{n}(x):= \max\{\hat{\theta}_{n}(x),1-\hat{\theta}_{n}(x)\}; \ \mbox{where}\qquad \hat{\theta}_{n}(x) \equiv \Theta( \hat{\varphi}_n(x)) := [1+e^{-\hat{\varphi}_{n}(x)}]^{-1}.
\end{equation}

\textbf{Updating knowledge states. }

The knowledge state $f_n$ is intended to capture all available information about $X^*$ given $\cD_n$. Since the true Bayesian posterior $g_{n}$ is not attainable due to unknown $p(\cdot)$, we notationally distinguish between the \textit{approximate} knowledge state $f_n$ and the true $g_n$~\eqref{eq:g_n}. For assimilating information,  we mimic the exact Bayesian updating from \cite{waeber2011bayesian} and use the batched knowledge state transition introduced in~\cite{rodriguez2017generalized}. Thus we take $f_{n+1} = \Psi(f_{n},x_{n+1},B_{n+1};\hat{p}_{n+1},a_{n+1})$ with
	\begin{equation}
	\label{eq:batched_updating_pba}
 \Psi(f_{n},x_{n+1},B_{n+1}; p ,a)(u) {\propto}\left\{
	\begin{array}{l}
	\left[ p(x_{n+1})^{ B_{n+1}} (1-p(x_{n+1}))^{a-B_{n+1}} \right]f_{n}(u),\ x_{n+1}<u\\
	 \\
	 \left[(1-p(x_{n+1}))^{ B_{n+1}} p(x_{n+1})^{a-B_{n+1}} \right]f_{n}(u),\  x_{n+1}\geq u.
	\end{array}
	\right.
	\end{equation}
Note that we replace the unknown $p(x_{n+1})$ with the surrogate-based $\hat{p}_{n+1}(x_{n+1})$. Over multiple rounds, this implies that $f_n$ depends on the historical estimates $\hat{\varphi}_{1:n}(x_{1:n})$ introducing a complex path-dependency between the latest knowledge state and the past surrogates of $p(\cdot)$.
	
	\textbf{Sampling strategies. } A sampling policy $\eta$ is a rule  which maps knowledge states to actions, namely sampling decisions.  The sampling decision to be made at step $(n+1)$ concerns the new query site $x_{n+1}$ and the respective number of replicates $a_{n+1}$. We consider two complementary ideas: (i) first select $a_{n+1}$ and then $x_{n+1}$; (ii) choose $x_{n+1}$ and then determine the respective $a_{n+1}$.
	
Approach (i) utilizes fixed replication amount $a\ge 1$ and selects the new $x_{n+1}$ using an information-theoretic criterion. In analogy to the Information Directed Sampling (IDS) policy used in the G-PBA context~\citep{rodriguez2017generalized}, we consider a criterion based on the \textit{batched} expected Kullback-Leibler (KL)  divergence $\E[ D(f_{n+1};f_n)]$ between $f_{n} $ and the updated knowledge state $f_{n+1}=\Psi(f_{n},x,B_{n+1};\hat{p}_{n+1},a)$,
\begin{equation}
		\label{eq:BatchedInformationCriterion0}
		\cI(x,f_{n};\hat{p}_{n}(x),a) := \mathbb{E}^{B(x)}_{\hat{p}}\left[\int_{0}^{1} \log_{2}\left(\frac{f_{n}(u|B(x),a)}{f_{n}(u)}\right)f_{n}(u)du\right];
\end{equation}
where the expectation is taken with respect to the random variable $B(x)\sim \mathsf{Bin}(a,\hat{p}_{n}(x))$ {and $\hat{p}_{n}(\cdot)$ is recovered using~\eqref{eq:Estimated-pn}}. Given the acquisition function~\eqref{eq:BatchedInformationCriterion0}, the next sampling location is its greedy maximizer
	\[
	x_{n+1}^{\mbox{\tiny sIDS}}:= \argsup_{x \in (0,1)} \cI(x,f_{n};\hat{p}_{n}(x),a).
	\]

The IDS rule was shown to be optimal for the base case of known and constant $p(x)$ and $a=1$~\citep{jedynak2012twenty}. In that case it is known to correspond to selecting $x_{n+1}$ which maximizes the conditional mutual information between the oracle response $Y_{n+1}(x_{n+1})$ and $X^{*}$ given $f_{n}$. Approach (ii), dubbed Randomized Quantile Sampling (RQS), selects locations using the knowledge state $f_{n}$ as a proposal density, i.e., $x_{n+1}^{\mbox{\tiny RQS}} \sim f_{n}(\cdot)$ and then adaptively picks  $a_{n+1}$.
The RQS strategy resembles Thompson sampling~\citep{russo2014information} and was shown to be competitive with IDS (and frequently slightly better) in the earlier G-PBA context. Conditional on $x_{n+1}$, $a_{n+1}$ is then picked to control the surrogate accuracy at $x_{n+1}$ in order to ensure the right amount of learning.

\textbf{Estimating the root $X^{*}$. }The final ingredient is the rule $\hat{x}_{n}$ to construct a point estimate of the root $x^*$ based on $f_n$. In analogy to the classical PBA setting \citep{waeber2013bisection}, we utilize the posterior median which we find is generally more robust than say the mean, as $f_n$ is often skewed or multi-modal,
\begin{equation}
\label{eq:median_fn}
\hat{x}_{n} := \mbox{median}(f_{n}).
\end{equation}	

\subsection{Binomial Gaussian Process Regression}
	\label{sec:BinomialGPs}
GPs can conveniently be used to specify prior distributions for Bayesian inference in the regression context and are widely adopted for sequential design tasks. In G-PBA they facilitate managing the sample budget for calling~\eqref{eq:z-oracle} by quantifying the predictive uncertainty at the next sampling site $x_{n+1}$ in terms of the  number of replicates $a_{n+1}$~\citep{kaminski2015method,binois2017replication}. The related Binomial Gaussian processes (B-GPs) (also known as GP classification and originally introduced in \cite{williams1998bayesian}) arise naturally in the context of \textit{latent variable} regression for \eqref{eq:ProbPositiveResponse}. In this case, the $\varphi$ is seen as a realization of a random process whose finite dimensional distribution follows a Multivariate Normal (MVN) distribution and whose spatial dependency is described by a (stationary) covariance function.

While for plain regression with Gaussian noise inference can be done in closed form, since for a given election of covariance kernel the posterior corresponds also to a GP~\citep{rasmussen2006gaussian}, this is no longer the case for B-GPs. Indeed,  since the binomial data likelihood is not conjugate to the Gaussian prior, exact inference is analytically intractable and therefore approximations to the predictive posterior must be conducted. One route summarized in~\cite{nickisch2008approximations} is based on approximating the non-Gaussian posterior with a tractable Gaussian distribution. Some of the most common instances of such schemes are the Laplace Approximation (LA)~\citep{williams1998bayesian} and Expectation
Propagation~(EP)~\citep{minka2001expectation}.

Let us assume that the surrogate $\varphi$ in~\eqref{eq:SurrogateModel0} is drawn from a GP prior, $\varphi \sim \mathsf{GP}(0,\kappa_{\vartheta}(\cdot, \cdot))$, characterized by a \textit{covariance kernel} function $\kappa_\vartheta(\cdot,\cdot)$ and parameterized by a vector of \textit{hyperparameters} $\vartheta \equiv (\tau^2, l)$. One of the most commonly used kernels is the Mat\'ern-$5/2$ family,
\begin{equation}
\label{eq:matern52}
\kappa_\vartheta(x_{i},x_{j}):= \tau^{2}\left[1+ \sqrt{5}r/l + 5r^{2}/(3l^{2})\right]e^{-\sqrt{5}r/l} \quad r:=|x_{i} - x_{j}|;
\end{equation}
where $\tau^{2}\ge 0$ is the intrinsic GP variance, and $l>0$ is the length-scale, which governs how fast the correlation decreases as the distance $|x_{i} - x_{j}|$ between inputs increases.

\textbf{Binomial GPs as latent variable models. } For fixed hyper-parameter $\vartheta$, the joint distribution of the vector $\varphi_{1:n}:= (\varphi_{1}(x_{1}),\ldots,\varphi_{n}(x_{n}))$ is a MVN
\begin{equation}
\label{eq:JointPrior-Z}
\varphi_{1:n} \sim \mathsf{N}(\bm{0},\bm{K}_{n}),
\end{equation}
where $\mathbb{E}[\varphi_{1:n}|x_{1:n}] = \bm{0}$ is the mean vector and $\bm{K}_{n}\equiv Cov(\varphi_{1:n}|x_{1:n})$ is the covariance matrix with entries $\kappa_{\vartheta}(x_{i},x_{j})=Cov(\varphi_{i},\varphi_{j}|x_{i,j})$. Inference of $\theta(\cdot)$ in \eqref{eq:ProbPositiveResponse} is conducted in two stages. First, we compute the posterior distribution of the vector $\varphi_{1:n}$ given the training data {$\mathcal{D}_{n}:=(B_{1:n},a_{1:n})$, consisting of the history of binomial $B_{i}$ responses and number of queries $a_{i}$ at each location $x_{i}$},
\begin{equation}
\label{eq:JointPosterior-Z}
p(\varphi_{1:n}|\cD_n) \propto p(B_{1:n}|\varphi_{1:n},a_{1:n})p(\varphi_{1:n});
\end{equation}
which is proportional to the binomial data likelihood $p(B_{1:n}|\varphi_{1:n},a_{1:n})$ times the MVN prior $p(\varphi_{1:n})$ given by~\eqref{eq:JointPrior-Z}. Second, the posterior predictive distribution $\varphi_{*}\equiv \varphi_{*}(x)$ at a location $x\in (0,1)$ is
\begin{equation}
\label{eq:PredictivePosterior-Z}
p(\varphi_{*}|\cD_n):= \int p(\varphi_{*},\tilde{\varphi}_{1:n}|\cD_n,x) d\tilde{\varphi}_{1:n},
\end{equation}
which is calculated by marginalizing the distribution of $\varphi_{*}$ over the joint posterior distribution of $(\varphi_{1:n},\varphi_{*})$ given by \eqref{eq:JointPosterior-Z}. Finally, the predicted $\hat{\theta}^{GP}(x)$ is produced by averaging the inverse link function with respect to \eqref{eq:PredictivePosterior-Z}; i.e.,  $\hat{\theta}_{n}^{GP}(x) :=\int (1+e^{-\varphi_{*}})^{-1}\cdot p(\varphi_{*}|\cD_n) d\varphi_{*}$.

\begin{remark}
	Following the classical inference paradigm for binomial regression we assume that $\theta(\varphi(x_{i}))$ is related to the random variable $\varphi(x_{i})$ via the canonical \textit{logistic} link function~\eqref{eq:SurrogateModel0}. Although other link functions can be entertained (such as the probit link), we use the logistic one since this link is used to obtain closed-form expressions for adaptive replication (see Lemma~\ref{lemma:CanonicalLink} in Section~\ref{sec:AdaptiveSampling}).
\end{remark}

The main challenge in computing the joint posterior \eqref{eq:JointPosterior-Z} is that the MVN prior over $\varphi_{1:n}$ does not correspond to a conjugate prior for the binomial likelihood, so either analytic approximations of integrals or solutions based on MCMC sampling are required. A commonly used method is to approximate the non-Gaussian posterior $p(\varphi_{1:n}|\cD_n)$ with a Gaussian one via Laplace Approximation.

\textbf{Laplace Approximation. }The Laplace method is constructed from the second order Taylor expansion of the \textit{score function}, ${\cal L}(\varphi_{1:n}) := \log p(\varphi_{1:n}|\cD_n)$, around its mode:
\[
\hat{\bm{\varphi}}_{n} = \argmax_{\bm{\varphi}_{n}} p(\bm{\varphi}_{n}|\cD_n).
\]

In Appendix~\ref{appendix:LaplaceApproximation} we show that this method yields a MVN approximation:
\begin{equation}
\label{eq:ApproxJointPosterior-Z}
p(\cdot| \cD_n) \simeq q(\cdot| \cD_n,\hat{\bm{\varphi}}_{n}) = \mathsf{N}(\cdot;\hat{\bm{\varphi}}_{n},(\bm{K}_{n}^{-1} + \hat{\bm{W}}_{n})^{-1}),
\end{equation}
where
\begin{equation}
\label{eq:EstimatedPosteriorMode}
\hat{\bm{\varphi}}_{n}:=(\hat{\varphi}_{1;n},\ldots,\hat{\varphi}_{n;n})
\end{equation} is found numerically via Newton-Raphson iterations using the training data $\mathcal{D}_{n}$, and $\hat{\bm{W}}_{n}$ is the $n \times n$ Fisher Information matrix of the binomial (negative) log-likelihood $l(\varphi_{1:n}):=\log p(B_{1:n}|a_{1:n},\varphi_{1:n})$. Importantly, if the canonical link is used, then the $i$-th entry of $\hat{\bm{W}}_{n}$ corresponds to the variance of the binomial response $B_{i}$ at $x_{i}$:
\begin{lemma}
	\label{lemma:CanonicalLink}
	Under the Bernoulli link function~\eqref{eq:SurrogateModel0}, the Hessian $\hat{\bm{W}}_{n}(\varphi_{1:n})=-\Delta l(\varphi_{1:n})$
	is diagonal:
	\begin{equation}
	\label{eq:HessianBinomialLikelihood}
	w_{ij} = \left\{
	\begin{array}{ll}
	a_{i}\Theta(\varphi_{i;n})(1-\Theta(\varphi_{i;n})), &i=j, \\
	0 & i\neq j, \quad \mbox{for $i,j = 1,\ldots,n$.}
	\end{array}\right.
	\end{equation}
\end{lemma}
 Hence, we have that $\hat{\bm{W}}_{n} = \mbox{diag}(\hat{w}_{1;n},\ldots,\hat{w}_{n;n})$; where $\hat{w}_{i;n}:=a_{i}\Theta(\hat{\varphi}_{i;n})(1-\Theta(\hat{\varphi}_{i;n}))$ are evaluated at the posterior mode~\eqref{eq:EstimatedPosteriorMode}. Having found the joint~\eqref{eq:ApproxJointPosterior-Z}, the (approximated) predictive posterior density $\varphi_{*} \sim \mathsf{N}(m_n(x), s_n^2(x))$ is also Gaussian  with mean $m_{n}(x)\equiv m_{n}(x;\hat{\bm{\varphi}}_{n})$ and posterior variance $s_{n}^{2}(x)\equiv s_{n}^{2}(x;\hat{\bm{\varphi}}_{n})$:
\begin{subequations}\label{eq:PosteriorPredictive}
	\begin{align}
	\label{eq:PosteriorPredictiveMean-n}
	m_{n}(x) &:= \bm{K}_{n}^{T} \bm{K}_{n}^{-1} \hat{\bm{\varphi}}_{n}; \\
	\label{eq:PosteriorPredictiveVariance-n}
	s_{n}^{2}(x) &:=  \bm{\kappa}_{n}^{T} (\bm{K}_{n} + \hat{\bm{W}}_{n}^{-1})^{-1}\bm{\kappa}_{n},
	\end{align}
\end{subequations}
where $\bm{\kappa}_{n}:= (\kappa(x,x_{1}),\ldots,\kappa(x,x_{n}))^{T}$ is the $n\times 1$ vector of covariances between $\varphi_{*}$ and $\varphi_{1:n}$.
The resulting point estimate for $\theta(x)$ {given $\mathcal{D}_{n}$} is thus
\begin{equation}
\label{eq:theta-hat}
\hat{\theta}^{GP}_{n}(x):=\int_{\mathbb{R}} (1+e^{-\varphi_{*}})^{-1} \mathsf{N}(\varphi_{*};m_{n}(x),s_{n}^{2}(x)) d\varphi_{*}, \quad x\in (0,1).
\end{equation}
Numerically, $\hat{\theta}^{GP}_n(x)$ is obtained by approximating the integral in~\eqref{eq:theta-hat} via a quadrature method. In particular we use \texttt{integrate()} which is part of the core distribution of \texttt{R} and relies on the Gauss-Kronrod quadrature method~\citep{RSoftware}.

\textbf{Hyper-parameter estimation. } The above model specification is valid for fixed hyperparameters $\vartheta$. To optimize the latter, we  consider a maximum a posteriori estimate (MAP),  $\hat{\vartheta}:= \argmax_{\vartheta} \{\log q(\cD_n|\vartheta)+  \log q_{0}(\vartheta)\}$ based on a prior $q_0(\cdot)$. In order to obtain $\hat{\vartheta}$ we use the package \texttt{GPstuff}~\citep{vanhatalo2013gpstuff}, which uses interleaved numerical optimization: at iteration $m$ given $\hat{\vartheta}^{(m)}$, evaluate the covariance matrix $\bm{K}_n(\hat{\vartheta}^{(m)})=(\kappa_{\hat{\vartheta}^{(m)}}(x_{i},x_{j}))_{i,j=1}^{n}$ and so estimate the mode $\hat{\bm{\varphi}}_{n}^{(m)}$;  then fix $\hat{\bm{\varphi}}^{(m)}_{n}$ and find $\hat{\vartheta}^{(m+1)} = \arg \max_\vartheta \log q(\cD_n|\vartheta,\hat{\bm{\varphi}}^{(m)}_{n})+ \log q_{0}(\vartheta)$, where  $q(\cD_n|\vartheta,\hat{\bm{\varphi}}^{(m)}_{n})$ is the data marginal log-likelihood,
\[
\log q(\cD_n|\vartheta,\hat{\bm{\varphi}}_{n}) = -\frac{1}{2} \hat{\bm{\varphi}}_{n}^{T}\bm{K}_{n}(\vartheta)^{-1}\hat{\bm{\varphi}}_{n} + \log p(B_{1:n}|a_{1:n},\hat{\bm{\varphi}}_{n}) -\frac{1}{2}\log \{|\bm{K}_{n}(\vartheta)|\cdot|\bm{K}_{n}(\vartheta)^{-1} + \hat{\bm{W}}_{n}(\hat{\bm{\varphi}}_{n})|\},
\]
which is available in closed-form, {see Algorithm 5.1 in~\cite{rasmussen2006gaussian}}.

\subsection{Adaptive Batching using the Posterior GP Variance}
\label{sec:AdaptiveSampling}
The posterior variance $s_n(\cdot)$ of the surrogate quantifies the quality of learning the latent GP. It can be used to guide sampling decisions via the associated information gain regarding $\varphi(\cdot)$. This is achieved by considering the look-ahead $s_{n+1}(\cdot)$ conditional on sampling at $x_{n+1}$. For plain GPs, $s_{n+1}$ is independent of the future response and hence can be evaluated exactly.
Unfortunately, for binomial GPs the look-ahead predictive variance \emph{does} depend on the future $B_{n+1}(x_{n+1})$. Specifically, Equation~\eqref{eq:PosteriorPredictiveVariance-nplus1} expresses the fact that $s_{n+1}^{2}(x_{n+1})$ depends on the entire $\hat{\bm{\varphi}}_{n+1}$ (computed based on $\mathcal{D}_{n+1}$).

\begin{theorem}
	\label{thm:PosteriorPredictiveVariance}
	The look-ahead variance $s_{n+1}^{2}(x_{n+1})$ at a new location $x_{n+1}$ under the Laplace approximation \eqref{eq:PosteriorPredictive} is given by
	\begin{align}
	\label{eq:PosteriorPredictiveVariance-nplus1}
	s_{n+1}^{2}(x_{n+1}) &= \left(
	\frac{1}{s_{n}^{2}(x_{n+1};\hat{\bm{\varphi}}_{1:n,n+1})} + \frac{1}{a_{n+1}\cdot \Theta(\hat{\varphi}_{n+1,n+1})(1-\Theta(\hat{\varphi}_{n+1,n+1}))}
	\right)^{-1} \\
\label{eq:PosteriorPredictiveVariance-nplus2}
& \simeq \left(
\frac{1}{s_{n}^{2}(x_{n+1})} + \frac{1}{a_{n+1}\hat{\theta}^{GP}_{n}(x_{n+1})(1-\hat{\theta}^{GP}_{n}(x_{n+1}))}
\right)^{-1},
	\end{align}
where $s_{n}^{2}(x_{n+1})$ is the iteration-$n$ posterior variance from~\eqref{eq:PosteriorPredictiveVariance-n} and  $\hat{\theta}^{GP}_{n}(x_{n+1})$ is from~\eqref{eq:theta-hat}.

\end{theorem}

The approximation in \eqref{eq:PosteriorPredictiveVariance-nplus2} aims to remove the dependence of \eqref{eq:PosteriorPredictiveVariance-nplus1}  on $B_{n+1}$ by using only information available at iteration $n$. To do so, we approximate the denominator of the first term in~\eqref{eq:PosteriorPredictiveVariance-nplus1} via
$s_{n}^{2}(x_{n+1};\hat{\bm{\varphi}}_{n})\simeq s_{n}^{2}(x_{n+1};\hat{\bm{\varphi}}_{1:n;n+1}),
$
that is, using the estimated posterior mode at time $n$. Similarly, the future {local} binomial variance in the second term of~\eqref{eq:PosteriorPredictiveVariance-nplus1} is approximated by its iteration-$n$ counterpart $a_{n+1}\hat{\theta}^{GP}_{n}(x_{n+1})(1-\hat{\theta}^{GP}_{n}(x_{n+1}))$; see the full proof in the Appendix.

The look-ahead  variance forms the basis of numerous \emph{expected improvement} (EI)  design heuristics that quantify the gain from sampling at $x_{n+1}$, see e.g.~\cite{JonesSchonlauWelch98,ChevalierPicheny13}. Below we adapt these concepts to the setting of {binomial} GPs by  quantifying the approximate reduction in posterior variance of $\varphi(x_{n+1})$ due to sampling $a_{n+1}$ replicates at  $x_{n+1}$ and hence allowing optimization of $a_{n+1}$ conditional on $x_{n+1}$. Related batched EI criteria have recently appeared in \cite{kaminski2015method}; see also \cite{binois2016practical}.

The idea of adaptive replication is to aim for driving the iteration-$n+1$ variance $s_{n+1}^{2}(x_{n+1}) \le \nu_n$ below a threshold  $\nu_{n}$.
Using the variance decomposition formula in the RHS of~\eqref{eq:PosteriorPredictiveVariance-nplus2} and solving for $a^\nu_{n+1}$ we have that:
\[
a_{n+1}^{\nu} \ge \frac{1}{\hat{\theta}_{n}(x_{n+1})(1-\hat{\theta}_{n}(x_{n+1}))}\cdot \left(\frac{1}{\nu_{n}} - \frac{1}{s_{n}^{2}(x_{n+1})}\right).
\]
We therefore consider the following adaptive replication scheme:
\begin{equation}
\label{eq:GP-AdaptiveBatching}
\hat{a}_{n+1}^{\nu} := a_{0}^{\nu}\cdot 1_{\{s_{n}^{2}(x_{n+1}) < \nu_{n}\}} +\frac{1}{\hat{\theta}_{n}(x_{n+1})(1-\hat{\theta}_{n}(x_{n+1}))}\left(\frac{1}{\nu_{n}}-\frac{1}{s_{n}^{2}(x_{n+1})}\right)\cdot 1_{\{s_{n}^{2}(x_{n+1}) \geq  \nu_{n}\}}.
\end{equation}

\begin{remark}
	We focus on the predictive uncertainty in the latent process $\varphi$ as a measure to determine $a_{n+1}$ ---as opposed to the predictive variance of the random variable $\theta(\varphi(x_{n+1}))$. Focusing on the uncertainty of the latent GP is a common strategy in sequential design (especially when the data likelihood is Gaussian), see for example~\cite{ankennman:nelson:staum:2010,chen2017sequential}. Another common measure for constructing sequential designs is the posterior predictive entropy~\citep{kapoor2007active} which is the preferred uncertainty measure in the active learning framework.
\end{remark}

\subsection{MLE-Based Binomial Regression}

An alternative approach to B-GPs is to fit a linear surrogate of the form $\varphi(x):= \bm{\beta}^{T}\bm{\phi}(x)$ for a given set of \textit{basis functions}. Thus we seek the best fit in the function space $\mathcal{H} = \text{span}( \phi_j : j = 1,\ldots, p)$.  The coefficients $\bm{\beta} \in \mathbb{R}^{p}$ can be found by optimizing the \textit{penalized} binomial log-likelihood criterion
\begin{equation}
\label{eq:PenalizedNonParametricRegression}
\min_{\bm{\beta}} \sum_{i=1}^{n} \Bigl\{ B_{i}\sum_{j=1}^p \beta_j \phi_j(x_{i}) + a_{i}\log \left(1+\exp \Bigl(\sum_{j=1}^p \beta_j \phi_j(x_{i}) \Bigr)\right)\Bigr\} + \frac{1}{2}\lambda \mathcal{J}( \sum_{j=1}^p \beta_j \phi_j),
\end{equation}
where $\mathcal{J}(\varphi)$ is a penalty functional. The above specification includes the classical \textit{logistic regression model} when the basis elements in $\mathcal{H}$ are monomials and $\lambda = 0$, which we also implemented with AIC-based selection of the degree of the polynomial.

\textbf{Kernel Logistic Regression (KLR). } Another  choice is the family of positive definite kernel functions $\phi_{j}(\cdot):=\kappa_{l_{j}}(\cdot;\xi_{j})$, where each basis element $\kappa_{l_{j}}(\cdot,\xi_{j})$ is indexed by a location parameter $\xi_{j}$ and a scale parameter $l_{j}$. The corresponding space of functions $\mathcal{H}$ is a Reproducing Kernel Hilbert Space  with penalty functional $\mathcal{J}(\varphi) = ||\varphi||_{2}^{2}= \bm{\beta}^{T} \bm{\Phi} \bm{\beta}$, where $\bm{\Phi}_{ij}  = \phi_{j}(x_{i})$. A popular choice is the  \textit{Gaussian radial kernel}:
\begin{equation}
\label{eq:klr-gaussian}
\kappa_{l}(x; \xi):= \exp\left(- \frac{|x - \xi |^{2}}{ l^2} \right).
\end{equation}
KLR behaves similarly to Support Vector Machines: data inputs are mapped to a space spanned by positive definite kernel functions, and the loss function being optimized are also similar~\citep{zhu2005kernel}. For our purposes, it is natural to use $\xi_i = x_i$, i.e.~a separate kernel function for each query location.

\textbf{Spline Logistic Regression (SLR). }A further commonly used functional space $\mathcal{H}$ is the B-spline basis where the $\phi_{j}$'s are piecewise continuous functions defined in terms of a set of \textit{knots}. Namely, an order-$P$ spline with knots $(\xi_{j})_{j=1}^p$ is a piecewise-polynomial of order $P$, and has continuous derivatives up to order $P-2$. The B-spline family takes $P=4$ and can be represented in terms of $p>0$ basis functions: $\phi_{1}(x) = 1$, $\phi_{2}(x) = x$, and for $j=2,\ldots, p$, $\phi_{j+1}(x) = d_{j}(x) - d_{j-1}(x)$, where
\begin{equation}
\label{eq:splines-logisticRegression}
d_{j}(x):= \frac{(x-\xi_{j})_{+}^{3} - (x-\xi_{p})^{3}_{+}}{\xi_{p}-\xi_{j}}, \quad j = 2,\ldots,p.
\end{equation}
The basis coefficients  $\bm{\beta}$ are fitted by penalizing the curvature of $\varphi(\cdot)$ using $\mathcal{J}(\varphi) = ||\varphi^{\prime\prime}||_{2}^{2}$.

\section{Sampling Policies}\label{sec:sampling}

To make sampling decisions based on the surrogate $\varphi$ and the information about $X^*$ contained in $f_{n}$ we investigate three types of policies.

\textbf{Batched Information-Directed Sampling. } Our first approach utilizes fixed replication $a\ge 1$ and selects the next $x_{n+1}$ to maximize the estimated {batched} expected KL divergence between the knowledge state at $T_n$ and $T_n+a$ as in~\eqref{eq:BatchedInformationCriterion0}, available in closed-form according to Theorem~\ref{thm:BatchedInformation}.

\begin{theorem}
	\label{thm:BatchedInformation}
Let $x\in(0,1)$ and $f_{n}$ be the current knowledge state with cumulative distribution function (CDF) $F_n(\cdot)$. The expected KL divergence, $\E[ D(f_{n+1};f_n)]$ between $f_{n+1}$ and $f_{n}$ from $a$ queries at $x$ is given by
	\begin{subequations}
		\label{eq:BatchedInformationCriterion1}
		\begin{align}
		\E[ D(f_{n+1};f_n)] & = \mathbb{E}\left[\log_{2}\left(\frac{(1-p(x))^{B}p(x)^{a-B}}{c_{n}(x,B)}\right)\right]F_{n}(x) \label{eq:BatchedInformationCriterion2} \\
		& \ + \label{eq:BatchedInformationCriterion3} \mathbb{E}\left[\log_{2}\left(\frac{p(x)^{\tilde{B}}(1-p(x))^{a-\tilde{B}}}{c_{n}(x,\tilde{B})}\right)\right](1-F_{n}(x)).
		\end{align}
	\end{subequations}
where the expected values~\eqref{eq:BatchedInformationCriterion2} and \eqref{eq:BatchedInformationCriterion3} are taken with respect to $B\sim \mathsf{Bin}(a,1{-}\theta(x))$ and $\tilde{B}\sim \mathsf{Bin}(a,\theta(x))$, respectively.
\end{theorem}

\begin{proof}
	By definition, the KL divergence between $f_n$ and $f_{n+1}$ is:
	\begin{equation*}
	D(f_{n+1};f_n) = \int_{0}^{1} \log_{2}\left(\frac{f_{n+1}(u)}{f_{n}(u)}\right)f_{n}(u)du.
	\end{equation*}
Since $\mathbb{P}_{p}(B = j|a,x,u):=	\mathsf{Bin}(j;a,1-\theta(x))1_{\{u\leq x\}} + \mathsf{Bin}(j;a,\theta(x))1_{\{u>x\}}$ and for $0<u<x^*$ we have that $f_{n+1}(u):= [(1-p(x))^{B}p(x)^{a-B}]f_{n}(u)/c_{n}(x,B)$, cf.~\eqref{eq:batched_updating_pba}, so taking expectations end up with
	\begin{align*}
	\E[ D(f_{n+1};f_n)] & = \int_0^x\mathbb{E}\left[\log_{2}\left(\frac{(1-p(x))^{B}p(x)^{a-B}}{c_{n}(x,B)}\right)\right]f_n(u) du \\
	& \qquad + \int_x^1\mathbb{E}\left[\log_{2}\left(\frac{p(x)^{\tilde{B}}(1-p(x))^{a-\tilde{B}}}{c_{n}(x,\tilde{B})}\right)\right] f_n(u) du
	\end{align*}
	which simplifies to \eqref{eq:BatchedInformationCriterion1}. Above
	$$	c_n(x,B):=\left[ (1-p(x))^{B(x)} p(x)^{a-B(x)} \right]F_{n}(x) +  \left[ p(x)^{B(x)} (1-p(x))^{a-B(x)} \right](1-F_{n}(x))$$
	is the normalizing constant of the updating~\eqref{eq:batched_updating_pba}.
\end{proof}

\begin{figure}[hbt]
	{
		\centering
		\includegraphics[width=0.8\textwidth]{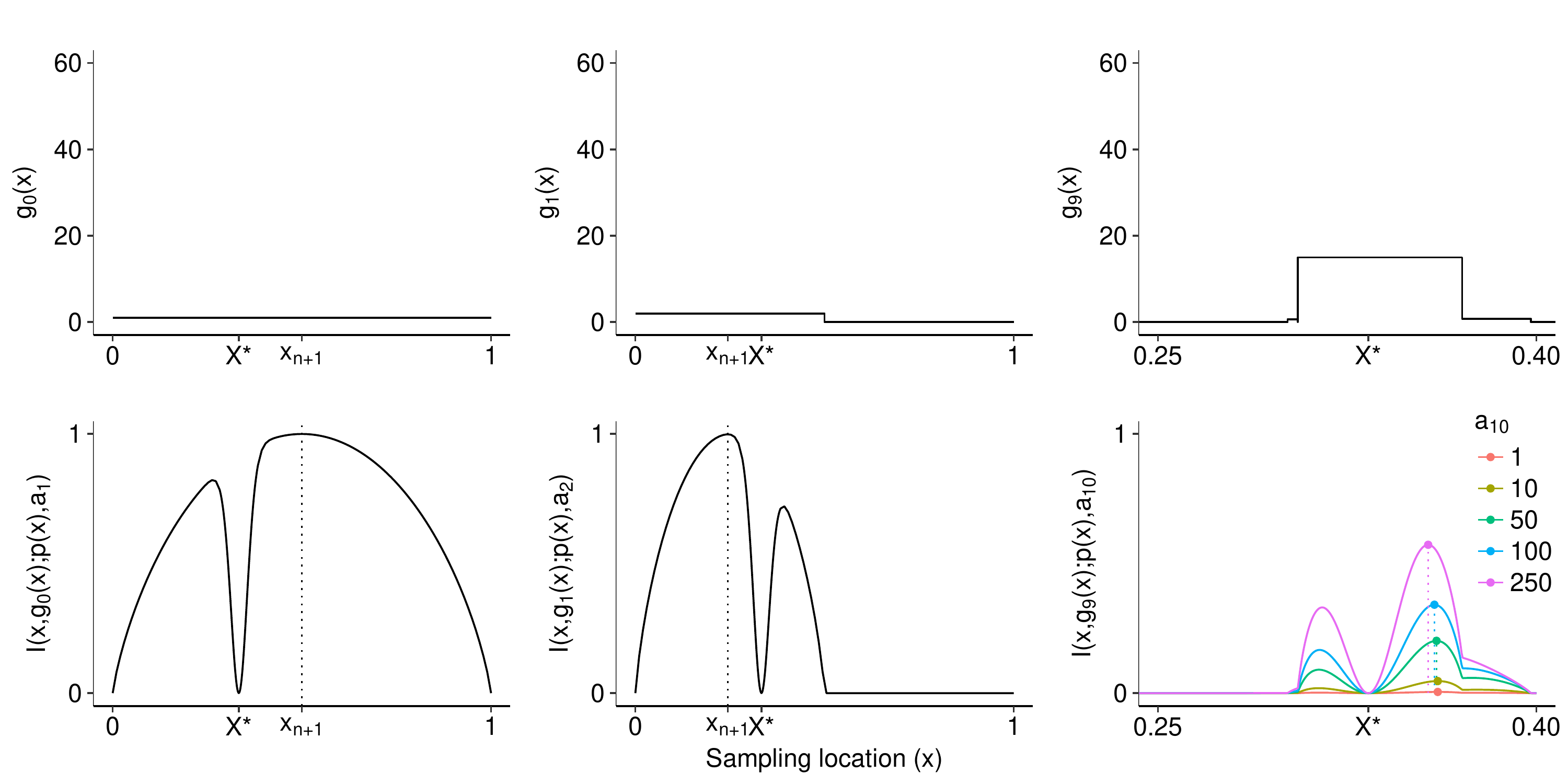}
		\caption{Data acquisition procedure using the batched information criterion $\cI(\cdot,g_{n};p(\cdot),a_{n+1})$ starting with a Uniform prior $g_{0}$ for the linear test function~\eqref{eq:g-ex1} and $x^{*}=1/3$. The first row shows the true Bayesian posterior $g_{n}$ for $n\in\{0,1,9\}$. The second row depicts the information gain function along with its maximizer $x_{n+1}$ (vertical dotted lines). The right-bottom plot shows the information criterion $\cI$ for several replication sizes $a_{10}\in\{1,10,50,100,250\}$ with the corresponding maximizers $x_{10}^{\mbox{\tiny sIDS}}$ of the information criterion $\cI$  (vertical dotted lines) given the knowledge state $g_{9}$ obtained by updating $g_{0}$ using fixed replication amounts $a_{1:9}=100$.
			\label{fig:MutualInformation2}}
	}
\end{figure}

{
 We now re-use KL divergence to define an acquisition function $\cI(x,f_{n};p(x),a) := \E_p[ D(f_{n+1};f_n)]$ as in \eqref{eq:BatchedInformationCriterion1} (emphasizing the dependence on the oracle accuracy $p$) and which is to be maximized over $x$.
To illustrate the relationship between the knowledge state $f_{n}$ and the batched information-criterion $x \mapsto \cI(x,f_{n};p(\cdot),a)$, Figure~\ref{fig:MutualInformation2} shows a realization of the Spatial IDS/PBA algorithm for a fixed batch size $a=100$ for $n=0,\ldots,9$ starting with $g_{0}\equiv \mathsf{Unif}(0,1)$ prior on $X^{*}$. The underlying response is~\eqref{eq:g-ex1} consisting of a decreasing linear function $h_{1}$ with root at $x^{*}=1/3$, and for now we assume access to the true oracle accuracy $p(x)=\Phi( 5|x-1/3|)$, so that the knowledge state is the exact Bayesian posterior $g_n$. We notice that sampling at $x^{\mbox{\tiny sIDS}}_{n+1}$ concentrates $g_{n}$ rapidly around the root $x^{*}$. Furthermore, $\cI$ typically has two local maxima, along with a global minimum at $x^{*}$ (sampling at the root is completely uninformative due to $p(x^*) = 0.5$). The right panel of Figure~\ref{fig:MutualInformation2} plots $x\mapsto \mathcal{I}(x,g_{{9}};p(x),a_{10})$ across different replication values $a_{10} \in \{1,10,5,100,250\}$. It can be seen that as the batch size $a$ is increased, information gain increases, but the maximizer $\argmax \cI(x,g_{9};p(x),a)$ (vertical dotted lines) does not change significantly. This is partly because the posterior $g_{{9}}$ is already concentrated.}

Crucially, maximizing \eqref{eq:BatchedInformationCriterion1} requires knowledge of the entire $x \mapsto p(x)$. This was one of the main challenges in the original G-PBA, where IDS was applied ad hoc \emph{after} estimating $p(\tilde{x}_i)$ at a set of $M \ge 2 $ candidate locations $\tilde{x}_{1:M}$. However, under our spatial modeling setting one can plug-in the surrogate $\hat{p}_{n}(x)$ and compute the maximizer of the resulting $\cI$ conditional on sampling $a_{n+1}\geq 1$ times at any $x\in (0,1)$. Thus, $x_{n+1}$ is chosen greedily as the maximizer of $\cI(\cdot, f_n, \hat{p}_n(\cdot), a_{n+1})$, that is,
\begin{equation}
\label{eq:Spatial-IDS}
x_{n+1}^{\mbox{\tiny sIDS}}:= \argmax_{x \in (0,1)} \cI(x,f_{n};\hat{p}_{n}(x),a_{n+1}).
\end{equation}

A numeric optimization procedure is needed to find $x_{n+1}^{\mbox{\tiny sIDS}}$. In our experiments below we utilize the \texttt{R} package \texttt{NLopt}~\citep{johnson2014nlopt}. In particular, we use the DIRECT (DIviding RECTangles) algorithm~\citep{jones1993lipschitzian} that implements gradient-free deterministic-search global optimization.

\medskip

The next two schemes switch the order, first picking $x_{n+1}$ and then $a_{n+1}$.

\textbf{Adaptive One-Step IDS policy.} Note that~\eqref{eq:Spatial-IDS} requires specifying the replication amount $a_{n+1}$. To implement the adaptive replication scheme~\eqref{eq:GP-AdaptiveBatching} within the IDS approach, we use an ad hoc heuristic which first maximizes $\cI$ using $a=1$ to get $x_{n+1}$ and then selects the replication amount
$a_{n+1}^{\nu}$. Let
\begin{equation}
\label{eq:OneStepIDS}
x_{n+1}^{\mbox{\tiny Ada-sIDS}} := \argmax_{x\in (0,1)} \ \cI(x;f_{n},\hat{p}_{n}(x),1).
\end{equation}
Conditional on $x_{n+1}$, $a_{n+1}$ is then picked to control the surrogate uncertainty at $x_{n+1}$ according to ~\eqref{eq:GP-AdaptiveBatching}. Observe that Ada-IDS is only feasible with a B-GP surrogate furnishing the predictive variance $s_{n}(x_{n+1})$.

\textbf{Randomized Quantiles Sampling. }The RQS strategy randomizes the next sampling location according to
\begin{equation}
\label{eq:randomized_policy}
x_{n+1}^{\mbox{\tiny RQS}} := F_{n}^{-1}(U_{n+1}), \quad \text{where}\quad  U_{n+1}\sim \mathsf{Unif}(0,1).
\end{equation}
The RQS policy can be interpreted as sampling based on the posterior distribution of $X^*$. This tends to sample close to the mean of $f_n$ but will also occasionally explore the latter's tails, capturing the trade-off between exploitation and exploration. An attractive feature of RQS is that it relies solely on $f_{n}$ so the surrogate $\theta_n$ is only used for updating $f_n$ in \eqref{eq:batched_updating_pba}.

\subsection{The Spatial Generalized Probabilistic Bisection Algorithm}

Summarizing the above developments, {Algorithm~\ref{alg:spatial-G-PBA} specifies the ingredients for blending surrogate modeling with probabilistic bisection. Two remarks are in order. First, the initialization step is non-sequential: we begin by employing $N_0 \times a_0 = T_{0} \ll T$ oracle evaluations to build $\hat{\varphi}_{N_0}$, picking equidistant (i.e.~space-filling) sites $x_{1:N_0}$ in $(0,1)$ and $a_0\ge 1$ replications per site. The corresponding $f_{T_0}$  is constructed via~\eqref{eq:batched_updating_pba}. Second, the surrogate re-fitting step in Algorithm \ref{alg:spatial-G-PBA} is user-controlled, since re-fitting can be expensive. In principle, re-fitting could be stopped entirely once $n$ is large enough, keeping the overhead cost of predicting $\theta_n(x)$ fixed, rather than increasing in $n$.} We also note that the chosen surrogates are non-sequential, i.e.~re-estimating $\hat{\varphi}_{n}$ is done from scratch, rather than via an updating formula (like is done for $f_n$).

\IncMargin{1em}
\begin{algorithm}
	\SetKwInOut{Input}{input}\SetKwInOut{Output}{output}
	\textbf{PBA parameters:} Prior $f_{0}$; $T_0$ and $a_{0}\ge 1$. Set $N_{0}:=T_{0}/a_{0}$\;
	\textbf{Surrogate initialization:} Regress $B_{1:N_{0}}$ on locations $x_{1:N_0}$ to obtain the surrogate model $\hat{\theta}_{N_{0}}$\;
	Update knowledge state starting from $f_{0}$ to $f_{T_{0}}$ given $\hat{\theta}_{N_{0}}$, $B_{1:N_{0}}$ and $x_{1:N_0}$\;
	$n\leftarrow N_{0}$, $T_{n}\leftarrow T_{0}$, $\mathcal{D}_{n} \leftarrow (B_{1:N_{0}},a_{1:N_{0}})$\;
	\While{$T_{n}<T$}{
		Using $f_n$ generate next sampling location $x_{n+1}$ and batch size $a_{n+1}$\;
		Query oracle $a_{n+1}$ times at $x_{n+1}$ to observe $B_{n+1}(x_{n+1})$\;
		\uIf{(OPTIONAL)}{
			Re-fit surrogate for $\hat{\theta}_{n+1}$ based on $\cD_{n+1} = (\cD_{n},B_{n+1},a_{n+1})$\;
		}
		\uElse{			
			$\hat{\theta}_{n+1}\leftarrow \hat{\theta}_{n}$\;
		}
		Update knowledge state at $x_{n+1}$ $f_{n+1}\leftarrow \Psi(f_{n},x_{n+1},B_{n+1};\hat{p}_{n+1}, a_{n+1})$ using $\hat{p}_{n+1} = \max\{\hat{\theta}_{n+1}(x_{n+1}),1-\hat{\theta}_{n+1}(x_{n+1})\}$\;
		Update $T_{n}\leftarrow T_{n} + a_{n+1}$ and $n\leftarrow n + 1$\;
	}
	\nllabel{alg:updating}
	\Return Knowledge state  $f_{N}$ and estimator for the root location $\hat{x}_{N} = \median(f_N)$\;
	\caption{Spatial Generalized-PBA.}\label{alg:spatial-G-PBA}
\end{algorithm}\DecMargin{1em}

\section{Numeric Examples}
\label{sec:NumericExamples}

We proceed to  empirically assess the performance of Algorithm~\ref{alg:spatial-G-PBA}. To do so, we mix-and-match the three components that the user must pick: the sampling policy $\eta$, surrogate model for $\hat{\theta}$  and the batch size $a$ (fixed or adaptive). To analyze the algorithm sensitivity to $(\eta,\hat{p},a)$, we consider multiple metrics regarding the quality of the root estimates, namely absolute residuals, credible interval length, and  corresponding coverage. Furthermore, we benchmark against schemes that are allowed to use the true posterior $g_n$ and $p(\cdot)$, quantifying the impact of learning the oracle. Our numeric examples are based on three test functions which capture different aspects and difficulties typically encountered in SRFPs, such as heteroscedasticity or zero curvature  at the root location.

\subsection{Experimental Setup}
\label{sec:ExperimentalSetup}

In analogy to \cite{waeber2013probabilistic,rodriguez2017generalized}, we utilize the  following three test functions $h_{i}(x)$ defined for $x \in (0,1)$:

	\begin{align} \text{linear}\qquad
	h_{1}(x) &= X^{*} -x ,  \qquad\qquad \sigma_1(x) = 0.2; \label{eq:g-ex1} \\
	\text{exponential} \quad
	h_{2}(x) &= e^{2(X^{*} - x)} - 1, \qquad \sigma_2(x) = 0.2 \cdot 1_{\{x < X^{*}\}} + 1 \cdot 1_{\{x > X^{*}\}};  \label{eq:g-ex2} \\
	\text{cubic}\qquad
	h_{3}(x) &= (X^{*} - x)^{3}, \qquad\quad \sigma_3(x) = 0.025. \label{eq:g-ex3}
	\end{align}

In all cases the stochastic simulator~\eqref{eq:z-oracle} consists of a Normally distributed $\epsilon(x){\sim} \mathsf{N}(0,\sigma_{i}^{2}(x))$ random noise and the root location $X^{*}\sim \mathsf{Unif}(0,1)$ is drawn from a Uniform distribution on $(0,1)$. We thus have that the ground-truth oracle $\theta(x)$ is given by $\theta_{i}(x):=\Phi(-h_{i}(x)/\sigma_{i}(x))$ for $i=1,2,3$.

 Figure~\ref{fig:testFunctions} displays the test functions $h_{i}(x)$  with $x^* = 1/3$ (first row), the maps $x \mapsto \theta_{i}(x)$ (second row), and the corresponding $\mbox{logit}(\theta_{i}(x))$ (third row) used for constructing $\hat{\varphi}$. The base example we investigate is the linear function \eqref{eq:g-ex1} whose slope is constant and significantly different from zero in locations close to the root $X^{*}$ and therefore leads to a simpler SRFP. In contrast, the curvature of \eqref{eq:g-ex2} together with the non-constant $\sigma_2(x)$ create a skew in the oracle and the posterior $f_n$.
Finally, example \eqref{eq:g-ex3} represents a difficult root-finding setting due to $h_{3}^{\prime}(X^{*}) = 0$, which implies that $p(x) \simeq 1/2$ in the vicinity of $X^{*}$.

\begin{figure}[hbt]
	{
		\centering
		\includegraphics[width=.6\textwidth]{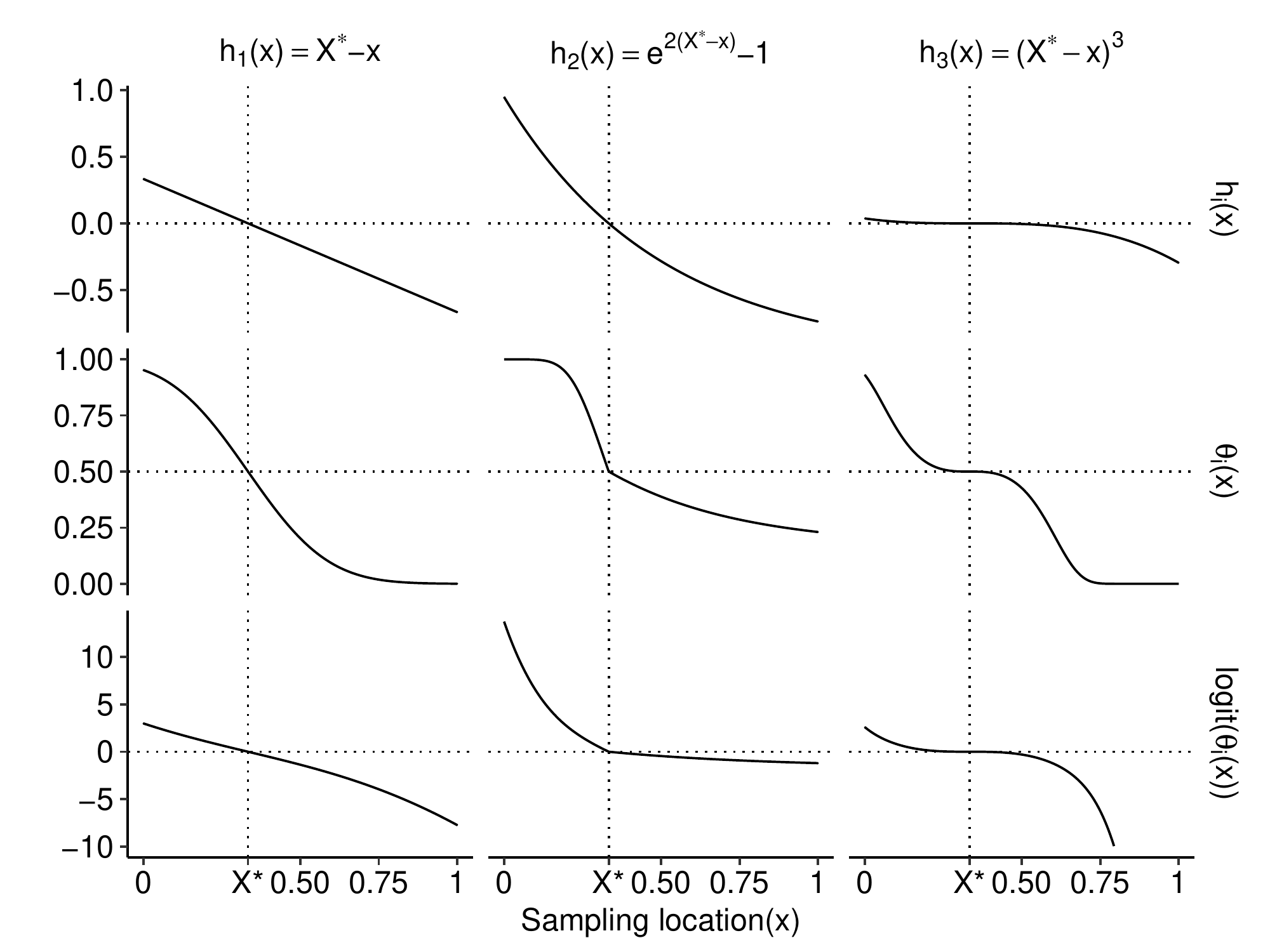}
		\caption[Graph of the three test functions used for numeric examples.]{Synthetic test functions \eqref{eq:g-ex1},~\eqref{eq:g-ex2} and~\eqref{eq:g-ex3} for Section~\ref{sec:NumericExamples}.
			\label{fig:testFunctions}}
	}
\end{figure}

\textbf{Performance Evaluation Metrics. }For a given configuration $(\eta,\hat{p},a)$ we use the following four performance metrics of the resulting Spatial G-PBA that all rely on $f_n$:
\begin{enumerate}
	\item \label{item:Perf1}\textit{Absolute residuals}: to determine the accuracy of the estimator $\hat{x}_n := \text{median}(f_n)$ we consider the $L_1$-residuals,
	\begin{equation}
	\label{eq:EvalMeasure-3}
	r(f_{n}) := |\hat{x}_{n} - x^{*}|;
	\end{equation}
	\item \label{item:Perf2} \textit{Credible intervals}: we evaluate the degree of uncertainty associated to the unknown root location $X^{*}$ through the length of a symmetric $(1-\alpha)$\% credible interval (CI) between the  $\alpha/2$ and $(1-\alpha/2)$ percentiles of $f_{n}$:
	\begin{equation}
	\label{eq:EvalMeasure-1}
	l_{1-\alpha}(f_{n}):= F_{n}^{-1}(1-\alpha/2) - F_{n}^{-1}(\alpha/2);
	\end{equation}
	
	\item \label{item:Perf3} \textit{Coverage}: to measure the accuracy of the above CI we evaluate
	\begin{equation}
	\label{eq:EvalMeasure-2}
	c_{1-\alpha}(f_{n}):=  \bm{Pr} \left\{x^{*} \in [F_{n}^{-1}(\alpha/2),F_{n}^{-1}(1-\alpha/2)] \right\},
	\end{equation}
	where the averaging in $\bm{Pr}\{ \cdot\}$ is across Monte-Carlo (MC) runs of the algorithm to capture the sampling distribution. If $c_{1-\alpha}(f_n) \ll (1-\alpha)$ the coverage test indicates that $f_n$ prematurely collapses or equivalently overstates its confidence about $X^*$. Small CI length $l_{1-\alpha}$ relative to residuals $r$ will lead to low coverage $c$.
	For both $c$ and $l$ we use $\alpha=0.05$.
	\item \label{item:Perf4}\textit{KL divergence}: given the chosen querying sites $x_{1:n}$,  we compare $f_{n}$ to the true posterior $g_{n}$ ({which is available for our three synthetic examples but not for the case-study in Section~\ref{sec:Results-AmericanOption}}) using  the KL divergence, $D(f_n;g_n)$. Since both $f_{n}$ and $g_{n}$ are updated at the same set of knots (sorted in increasing order)  $\tilde{x}_{1:n}$, we may write $g_{n}(x):= \sum_{j=1}^{n} g(\tilde{x}_{j-1})1_{x\in[\tilde{x}_{j-1},\tilde{x}_{j})}$ and $f_{n}(x):= \sum_{j=1}^{n} f(\tilde{x}_{j-1})1_{x\in[\tilde{x}_{j-1},\tilde{x}_{j})}$, with $\tilde{x}_{0}:=0$ and $\tilde{x}_{n}:=1$. We then obtain
	\begin{equation}
	\label{eq:KL-fn-gn}
	D(f_{n};g_{n}):= \sum_{j=1}^{n} \log\left(\frac{f(\tilde{x}_{j})}{g(\tilde{x}_{j})}\right)f(\tilde{x}_{j}) (\tilde{x}_{j}-\tilde{x}_{j-1}).
	\end{equation}	
	We make the usual convention that $\log(f(x)/g(x))f(x)=0$ if $f(x)=0$ (including when $g(x)=0$); as well as $\log(f(x)/g(x))f(x)=+\infty$ if $g(x)=0$ and $f(x)>0$~\citep{cover}. Practically, to estimate the average KL divergence we consider only finite values.
\end{enumerate}

Performance metrics \ref{item:Perf1}-\ref{item:Perf4} are averaged using a total of $MC=100$ Monte Carlo macro-iterations. To make all schemes comparable, we fix $X^{*}_{(i)}\sim\mathsf{Unif}(0,1)$ and each combination of $(\eta,\hat{p},a)$ is applied using the same root value $X^{*}_{(i)}$ during the $i$-th MC iteration, $i=1,\ldots,MC$.

\noindent\textbf{Surrogates for $p(\cdot)$: }

\begin{itemize}
	\item \textit{B-GP.} For the binomial GP (B-GP) we use the 5/2-Mat\'ern covariance kernel~\eqref{eq:matern52}. The hyper-parameters $\vartheta = (\tau^{2},l)$ are estimated via a Bayesian MAP estimation procedure, placing a square root uniform prior (i.e., $q_{0}(\sqrt{\tau^{2}})\propto 1)$ on $\tau^{2}$ and a Student-$t$ prior on the length scale parameter $l$ (both default priors for binomial GPs in \texttt{GPstuff}). Although parameter estimation can be expensive, the B-GP is \textit{re-fitted} and \textit{updated} every $T_{n}=a_{n}$ simulation outputs; that is, the hyper-parameters $\hat{\vartheta}$ are re-fitted and the posterior mode~$\hat{\bm{\varphi}}_{n}$ is re-computed every time a new pair of sampling location/binomial response is observed, such that the surrogate is able to assimilate acquired information.
	
	\item \textit{KLR. }Kernel Logistic Regression (KLR) is implemented with the Gaussian kernel basis function~\eqref{eq:klr-gaussian} using a  \textit{fixed} length scale parameter $l\equiv 1$ and centering $\phi_{j}$ at each sampling location $\xi_j \equiv x_{j}$, $j = 1,\ldots,n$ (implying that we use as many kernel functions as sampling points to learn $\varphi$). Since we would like to induce a surrogate model $\hat{\varphi}$ that closely resembles the local estimators $\hat{p}(\cdot)$, we use a (small) \textit{fixed} value  $\lambda = 0.01$ as the penalty parameter for optimizing~\eqref{eq:PenalizedNonParametricRegression}. Numerically, we implement KLR as stated in Algorithm 1 of~\cite{zhu2005kernel}.
	
	\item \textit{SLR: } We consider a \textit{smoothing} spline logistic regression (SLR) model where the penalty coefficient $\hat{\lambda}$ (aka smoothing parameter) is estimated via Generalized Cross-Validation~\citep{friedman2001elements} jointly with the spline basis coefficients. In this case, the spline knots $\xi_j$ are placed at percentiles of the sampling locations $x_{1:n}$. Thus, as the mass of $f_n$ concentrates around $x^{*}$ (and hence sampling locations $x_{1:n}$ concentrate around the root), more knots $\xi_{j}$'s are also placed near $X^{*}$,  making the surrogate more localized in regions where the variability of the binomial responses $B_{n}$ is maximal.
	
	\item \textit{LR. } Polynomial logistic regression with $\varphi(x) = \beta_{0} + \sum_{j=1}^{5}\beta_{j}x^{5}$, a {quintic} polynomial and zero penalty  $\lambda = 0$ (to enforce surrogate flexibility). Both the SLR and LR surrogates are implemented using the \texttt{gam()} routine from the \texttt{mgcv} package in \texttt{R}~\citep{wood2001mgcv}.
\end{itemize}

\noindent\textbf{Sampling Policies $\eta$: }
\begin{itemize}
	\item Spatial Information-Directed Sampling (sIDS)~\eqref{eq:Spatial-IDS};
	\item Spatial Randomized Quantile Sampling (sRQS)~\eqref{eq:randomized_policy};
	\item One-step sIDS~\eqref{eq:OneStepIDS} combined with the adaptive replication scheme $a_{n+1}^{\nu}$.
\end{itemize}

For the initialization stage in  Algorithm~\ref{alg:spatial-G-PBA} we use $N_{0}$ equally spaced  $x_{1:N_{0}}$ to learn $\varphi_{N_0}(\cdot)$ non-sequentially. In our experiments all surrogates are initialized using $T_{0}:=5000$ (i.e., 25\% of total sampling budget) oracle evaluations with  $a_{0} \in \{100,250\}$ which results in $N_{0}:=T_{0}/a_{0} \in \{50,20\}$ initial training locations.

\noindent \textbf{Adaptive Replication $a_{n+1}^{\nu}$. }The scheme~\eqref{eq:GP-AdaptiveBatching} has two parameters: the \textit{minimum replication amount} $a_0^\nu$ and the variance \textit{thresholding sequence} $(\nu_n)_{n\ge 1}$. In our experiments, we use  $a_0^\nu := 1$
in order to favor exploration in regions where the spatial surrogate $\varphi$ already learned $p(\cdot)$ sufficiently well, {as quantified in terms of the predictive posterior GP variance~\eqref{eq:PosteriorPredictiveVariance-n}}. For  $\nu_{n}$ we use the following two variants (see Algorithm~\ref{alg:spatial-G-PBA}):
\begin{equation}
\label{eq:v_n}
\nu_{n}^{(100)}:= 0.1/n \quad \text{ when } a_{0}=100 \quad \text{ and }\quad \nu_{n}^{(250)}:= 0.05/n \quad \text{when } \quad a_{0}=250.
\end{equation}
{This choice is linked to the fact that since the initialization stage budget $T_0$ is fixed, larger $a_0$ makes $N_0$ smaller and hence
leads to larger $s_n$, and so we take the thresholds $\nu_n$ larger as well.}
{  To avoid excessive batching which could occasionally arise in our implementation we bound $a_{n+1}^{\nu}$ (specifically by 1000 in all experiments). This allows to manage the overall sampling budget in order to enforce exploration.}

\textcolor{black}{Figure~\ref{fig:AdaptiveBatching} depicts the \emph{realized} replication amounts $n \mapsto a_{n+1}^{\nu}(x_{n+1})$ using the one-step sIDS policy~\eqref{eq:OneStepIDS} applied to our running example~\eqref{eq:g-ex1} (during initialization, $n\le N_{0}:=T_{0}/a_{0}$, $a_n \equiv a_0$ is fixed). We observe that  $a_{n+1}^{\nu}$ generally slowly decreases  as $n$ rises, although the local behavior can be quite ``spiky'': sometimes a large batch is required to bring $s_{n}^{2}(x_{n+1})$ below $\nu_{n}$, see top panels of Figure~\ref{fig:AdaptiveBatching}. One reason is that as $n$ increases, the sampling concentrates around $x^*$. Since this region quickly becomes well-explored, we usually obtain quite low $s_{n}^{2}(x_{n+1})$ making $a_{n+1}$ low as well.
To give a sense of the macro-time behavior, by $T=2 \cdot 10^4$ the median number of sampling locations is $N_T=291$ and $N_T=116$ (for mean replication amounts of 70 and 170 respectively) for the thresholding sequences $\nu_{n}^{(100)}$ and $\nu_{n}^{(250)}$ in~\eqref{eq:v_n}.}

\begin{figure}[ht]
	{
		\centering
		\includegraphics[width=0.75\textwidth]{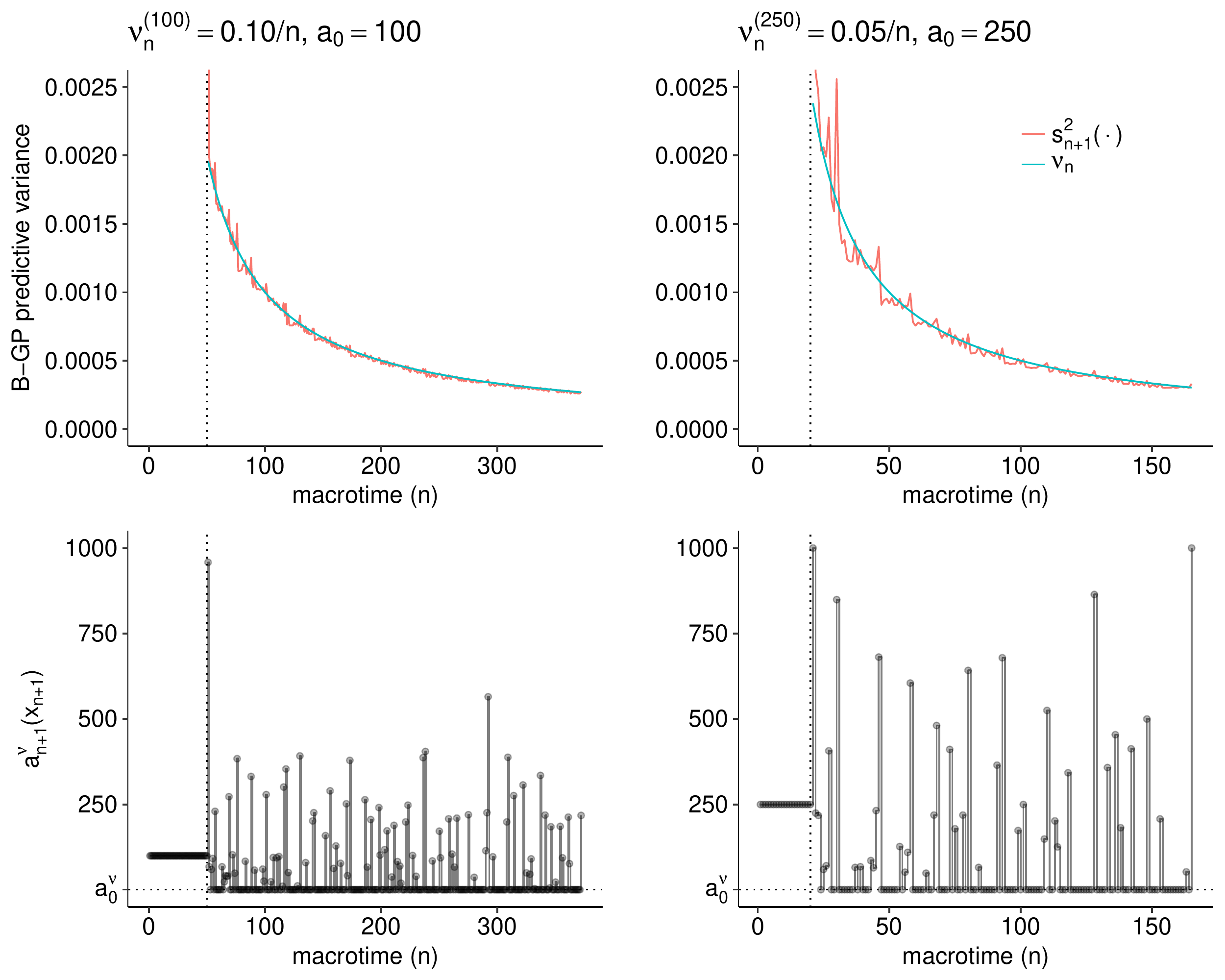}
		\caption[Adaptive replication scheme using Binomial Gaussian Processes.]{\textit{First row}: estimated predictive variance $s_{n}^{2}(x_{n+1})$ and the thresholding sequences $\nu_{n}^{(100)}=0.1/n$ (first column) and $\nu_{n}^{(250)}=0.05/n$ (second column), when the initial batch size $a_{0}$ for initializing the B-GP is $a_{0}=100,250$, respectively. \textit{Second row}: adaptive replication amount $a_{n+1}^{\nu}(x_{n+1})$ ($y$-axis) in macro-time  $n$ ($x$-axis) selecting $x_{n+1}$ using the one-step IDS criterion~\eqref{eq:OneStepIDS}.
			\label{fig:AdaptiveBatching}}
	}
\end{figure}

\subsection{Illustrating Spatial G-PBA}

Figure~\ref{fig:FinalSurrogate} compares the fitted surrogate models using a fixed dataset $\mathcal{D}_{N}^{\eta}$ (in order to remove the effect of the design and target surrogate accuracy) generated using two different sampling policies $\eta$: sIDS (first row) and sRQS (second row) implemented using the true posterior $g_{n}$ and the linear test function~\eqref{eq:g-ex1}. For all models a fixed batch size $a\equiv 100$ and $N=200$ total training locations is used. Figure~\ref{fig:FinalSurrogate} depicts three fundamental features of spatial G-PBA: (i) the sIDS strategy achieves lower posterior $X^{*}$-uncertainty relative to the sRQS policy, as seen in the narrower confidence bands depicted in the right panel; (ii)  the design of the sIDS strategy brackets the root, gradually squeezing the posterior $f_n$ towards $X^*$; and (iii) the spatial surrogates succeed in learning the true $\theta_{1}(x):=\Phi(-\frac{1/3-x}{0.20})$ especially around the root, cf.~the left panels of the Figure. As a result, root estimation is significantly improved and leads to reliable posterior CIs on the right panels of Figure~\ref{fig:FinalSurrogate}.

\begin{figure}[ht]
	{
		\centering
		\includegraphics[width=.75\textwidth]{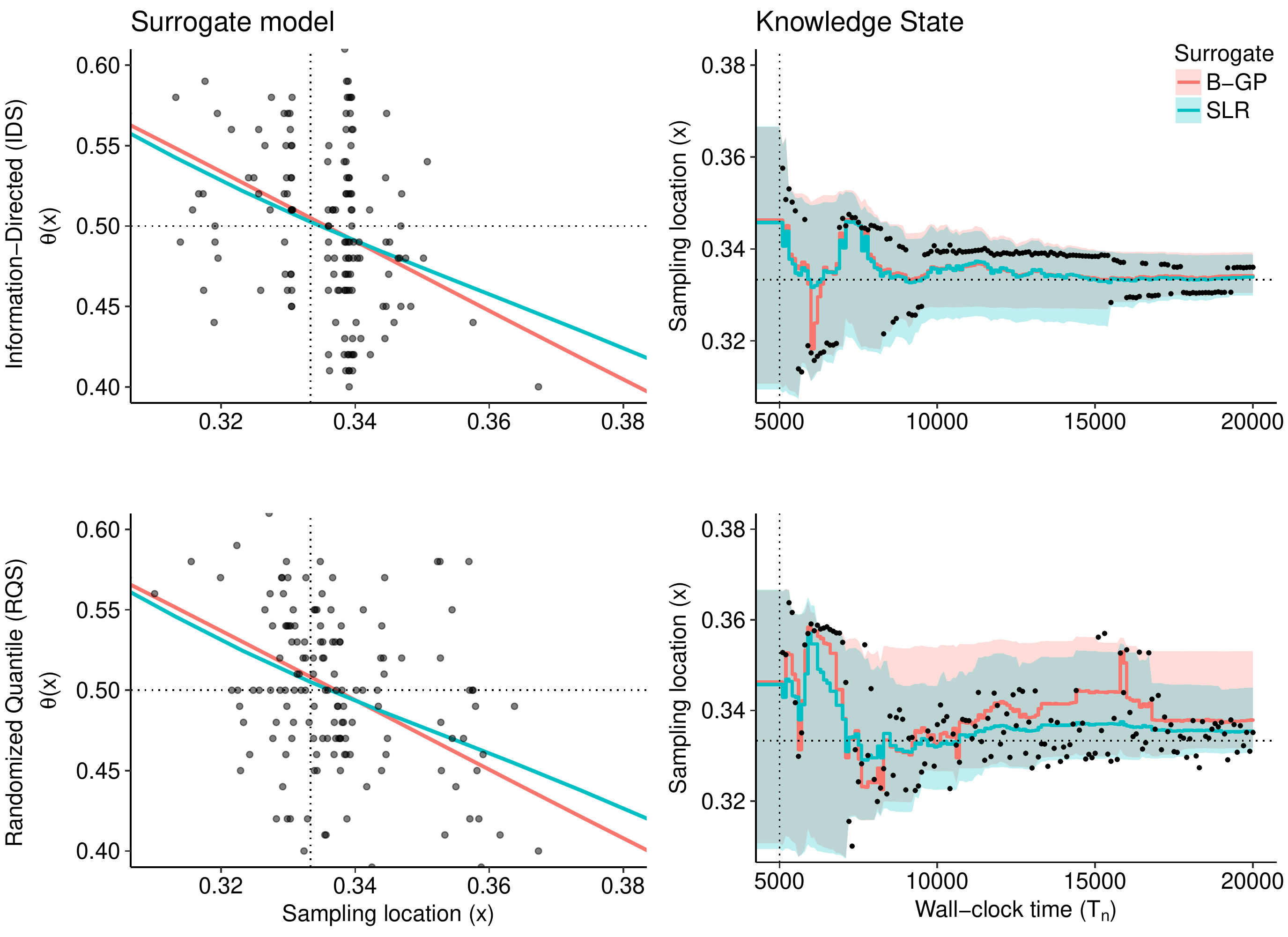}
		\caption[Comparison of surrogate models and sampling policies using spatial G-PBA.]{Spatial G-PBA with the linear test function~\eqref{eq:g-ex1}. \textit{Left}: B-GP, and SLR surrogates trained on a fixed dataset obtained using the sIDS policy (first row) and the sRQS policy (second row) at $T=2 \cdot 10^4$ and batch size $a=100$ (so that $N=200$). The $x$-axis is zoomed to the neighborhood of $x^* = 1/3$, so does not show the full $\cD_N$.  \textit{Right}: posterior inter-quantile range (shaded regions) across spatial surrogates (colors) and sampling policies (rows) as a function of $T_n$. We also show the corresponding root estimates $\hat{x}_n=\median(f_{n})$ (lines). The true $\theta(\cdot)$, as well as the estimated median using $g_{n}$ (i.e., knowledge state with the true $\theta(\cdot)$) is shown with a dashed line, respectively.
			\label{fig:FinalSurrogate}}
	}
\end{figure}

\subsection{Results}

\begin{table}[ht]
	\centering
	\caption[Performance metrics for the linear test function $h_{1}$ implementing the spatial G-PBA.]{Performance Monte-Carlo metrics  for the test function $h_{1}$ at $T=20,000$.
		\label{tab:results-h1}
	}
\resizebox{\columnwidth}{!}{%
	\begin{tabular}{c|c|rr|rr|rr|rr}
		\hline
		\multirow{ 2 }{*}{Policy $\eta$} & \multirow{ 2 }{*}{$\hat{p}$} & \multicolumn{2}{c|}{$\hat{r}(f_{T})$  ($10^{-2}$) }& \multicolumn{2}{c|}{$\hat{l}_{0.95}(f_{T})$  ($10^{-2}$) }& \multicolumn{2}{c|}{$\hat{c}_{0.95}(f_{T})$   }& \multicolumn{2}{c}{$\hat{D}(f_{T};g_{T})$}\\
		& & $a_{0}{=}100$ & $a_{0}{=}250$  & $a_{0}{=}100$ & $a_{0}{=}250$ &  $a_{0}{=}100$ & $a_{0}{=}250$ & $a_{0}{=}100$ &  $a_{0}{=}250$  \\
		\cline{1-10}\multirow{ 4 }{*}{ sIDS } & \multirow{ 1 }{*}{ B-GP } & 0.2241 & 0.1874 & 0.8931 & 0.9215 & 0.88 & 0.98 & 0.62 & 0.39 \\
		& \multirow{ 1 }{*}{ KLR } & 0.2106 & 0.2037 & 0.8998 & 0.9496 & 0.95 & 0.96 & 0.57 & 0.38 \\
		& \multirow{ 1 }{*}{ SLR } & 0.1864 & 0.1954 & 0.8669 & 0.8810 & 0.87 & 0.89 & 0.64 & 0.59 \\
		& \multirow{ 1 }{*}{ LR } & 0.1956 & 0.1709 & 0.8852 & 0.8708 & 0.94 & 0.98 & 0.56 & 0.38 \\
		\cline{1-10}\multirow{ 4 }{*}{ sRQS } & \multirow{ 1 }{*}{ B-GP } & 0.2230 & 0.1985 & 1.2683 & 1.3497 & 0.95 & 0.99 & 0.61 & 0.48 \\
		& \multirow{ 1 }{*}{ KLR } & 0.2152 & 0.1734 & 1.2052 & 1.3843 & 0.99 & 0.99 & 0.51 & 0.41 \\
		& \multirow{ 1 }{*}{ SLR } & 0.1935 & 0.2181 & 1.2027 & 1.2302 & 1.00 & 0.96 & 0.57 & 0.60 \\
		& \multirow{ 1 }{*}{ LR } & 0.1840 & 0.2012 & 1.2543 & 1.3174 & 0.96 & 0.97 & 0.56 & 0.50 \\
		\cline{1-10}\multirow{ 1 }{*}{ Ada-sIDS } & \multirow{ 2 }{*}{ B-GP } & 0.2016 & 0.2060 & 0.9730 & 1.0051 & 0.97 & 0.96 & 0.34 & 0.33 \\
		\multirow{ 1 }{*}{ Ada-sRQS } &  & 0.3025 & 0.2398 & 1.4612 & 1.5013 & 0.99 & 1.00 & 0.16 & 0.22 \\
		\hline\hline
		\cline{1-10}\multirow{ 2 }{*}{ Det-IDS } & \multirow{ 1 }{*}{ $\bar{p}$ } & 0.3692 & 0.2996 & 0.0196 & 0.0773 & 0.01 & 0.05 & 26.58 & 7.85 \\
		& \multirow{ 1 }{*}{ $\hat{p}_{\mathscr{L}_{0}}$ } & 0.4377 & 0.3576 & 0.0769 & 0.2068 & 0.03 & 0.13 & 21.78 & 6.40 \\
		\cline{1-10}\multirow{ 2 }{*}{ RQS } & \multirow{ 1 }{*}{ $\bar{p}$ } & 0.4422 & 0.2528 & 0.0000 & 0.0038 & 0.00 & 0.01 & 31.74 & 19.59 \\
		& \multirow{ 1 }{*}{ $\hat{p}_{\mathscr{L}_{0}}$ } & 0.4099 & 0.2735 & 0.0384 & 0.0202 & 0.01 & 0.03 & 32.60 & 17.63 \\	\hline
	\end{tabular}
}
\end{table}

Table~\ref{tab:results-h1} shows the results for the linear test function~\eqref{eq:g-ex1}. To allow a direct comparison to the non-spatial G-PBA, the last few rows present the performance of the best \textit{local} G-PBA schemes
as identified in~\cite{rodriguez2017generalized}:
\begin{itemize}
	\item  the \textit{empirical majority  proportion}, $\bar{p}(x_n):= \max\{B_n/a_n,1-B_n/a_n \}$; and
	\item  the \textit{posterior mode} given $\bar{p}(x_n)$, $\hat{p}_{L_{0}}(x):=\argmax_{p\in(1/2,1)}\pi(p|\bar{p}(x))$; where $\pi(\cdot|\bar{p}(x))$ is the posterior density of $p$ seen as a random variable with prior $\pi_{0} \equiv \mathsf{Unif}(1/2,1)$.
\end{itemize}
The local estimators $\bar{p}$ and $\hat{p}_{L_{0}}$ are then applied within two non-spatial G-PBA policies:
\begin{itemize}
	\item  Deterministic-IDS (Det-IDS) which chooses $x_{n+1}$ by maximizing  $\cI(\tilde{x}_{n,i};f_{n},\hat{p}(\tilde{x}_{n,i}), a_{0})$ among the two candidates $\tilde{x}_{n,i} \in  \{F_{n}^{-1}(0.25),F_{n}^{-1}(0.75)\}$ (i.e., the 25-th and 75-th quantiles of $f_{n}$);
	\item  Local RQS which selects $x_{n+1}$ according to \eqref{eq:randomized_policy}.
\end{itemize}

Table~\ref{tab:results-h1} demonstrates that using surrogate modeling substantially improves root estimation relative to the original G-PBA. Indeed, we obtain significantly lower residuals (roughly half as big), and narrower CI across while maintaining a high probability coverage.
The latter shows that the knowledge state $f_T$ is correctly converging to the true root value $X^{*}$. In particular, {polynomial} logistic regression (LR) offers a good choice as it minimizes the average residuals and length of CI, as well as matches the nominal coverage $\hat{c}(f_{T})\approx 0.95$, confirming that $f_T$ is close to the true posterior $g_T$.

Importantly, we can see that spatial modeling leads to a nearly two orders of magnitude reduction in the average KL divergence between $f_T$ and $g_T$, primarily due to the lower bias in the estimation of $p(\cdot)$ relative to the two local estimators considered. Thus, spatial G-PBA successfully resolves the problem of $f_T$ experiencing premature collapse which was a major concern in G-PBA where $c_\alpha$ was frequently unacceptably low. We note that $\hat{D}(f_{T};g_{T})$ is consistently low across all surrogate models $\varphi$, indicating that the goodness-of-fit for $\theta(\cdot)$ is not overly sensitive to the choice of the surrogate type.

In terms of the sampling policies, sIDS outperforms sRQS since the respective average residuals and CI length are lower while preserving a high coverage probability. For the replication regime $a_{n}$, we note a preference for $a=250$ (i.e., a total of $N=80$ design sites) which tends to yield better learning rates about $p(\cdot)$ (and therefore about $X^{*}$) compared to $a=100$, as measured by the average KL divergence. Adaptive batching generally under-performs, especially sRQS that frequently uses excessive batch sizes (see Fig.~\ref{fig:AdaptiveBatching}) far from the root, and hence does not exploit sufficiently. \textcolor{black}{At the same time, adaptive batching achieves the lowest KL divergence. It remains an open question how to best select the thresholding sequence.}

\textbf{Empirical results for the exponential and cubic test functions. }Tables~\ref{tab:results-h2} and~\ref{tab:results-h3} show the performance metrics for the test functions $h_{2}$ and $h_{3}$, respectively. Results are largely similar. As for $h_1$ we observe a large improvement in performance relative to non-spatial G-pBA, especially in terms of the coverage probability $\hat{c}$, \textcolor{black}{which was improved from $\hat{c}_{0.95} \approx 0$ (meaning the algorithm fails completely in providing a CI for $x^*$) to the actual nominal CI coverage value, see the right-most columns of Table~\ref{tab:results-h2} and Table~~\ref{tab:results-h3}. In terms of sampling policies, we note that sIDS again outperforms sRQS in terms of average absolute residuals and length of CI for both $h_2$ and $h_3$. Furthermore, polynomial logistic regression (LR) continues to be the best surrogate choice combined with {fixed} batch of $a_{n} \equiv 100$ (i.e., using a total of $N=200$ design points) implying a preference for \textit{exploration} in these harder problems}. We note that B-GP performs worse, especially for $h_{2}$, possibly due to the non-smoothness of $\theta_2(\cdot)$ at the root (cf.~Figure~\ref{fig:testFunctions}). Because B-GP assumes a smooth response surface it fails to properly capture such ``cusp'' that calls for a spatially non-stationary covariance structure.

\begin{table}[ht]
	\centering
	\caption[Performance metrics for the exponential test function $h_{2}$ implementing the spatial G-PBA.]{Performance Monte-Carlo metrics  for the test function $h_{2}$ at $T=20,000$.
		\label{tab:results-h2}
	}
\resizebox{\columnwidth}{!}{%
	\begin{tabular}{l|c|rr|rr|rr|rr}
		\hline
		\multirow{ 2 }{*}{Policy $\eta$} & \multirow{ 2 }{*}{$\hat{p}$} & \multicolumn{2}{c|}{$\hat{r}(f_{T})$  ($10^{-2}$) }& \multicolumn{2}{c|}{$\hat{l}_{0.95}(f_{T})$  ($10^{-2}$) }& \multicolumn{2}{c|}{$\hat{c}_{0.95}(f_{T})$   }& \multicolumn{2}{c}{$\hat{D}(f_{T};g_{T})$}\\
		& & $a_{0}{=}100$ & $a_{0}{=}250$  & $a_{0}{=}100$ & $a_{0}{=}250$ &  $a_{0}{=}100$ & $a_{0}{=}250$ & $a_{0}{=}100$ &  $a_{0}{=}250$  \\
		\cline{1-10}\multirow{ 4 }{*}{ sIDS } & \multirow{ 1 }{*}{ B-GP } & 0.5330 & 0.4439 & 1.1043 & 1.1666 & 0.45 & 0.53 & 1.61 & 1.41 \\
		 & \multirow{ 1 }{*}{ KLR } & 0.4537 & 0.4098 & 0.7974 & 0.8751 & 0.46 & 0.44 & 3.23 & 2.79 \\
		 & \multirow{ 1 }{*}{ SLR } & 0.4352 & 0.4076 & 1.0995 & 1.2555 & 0.67 & 0.80 & 1.21 & 0.86 \\
		 & \multirow{ 1 }{*}{ LR } & 0.3814 & 0.4128 & 1.0641 & 1.1795 & 0.60 & 0.56 & 2.08 & 1.61 \\
		\cline{1-10}\multirow{ 4 }{*}{ sRQS } & \multirow{ 1 }{*}{ B-GP } & 0.4817 & 0.5162 & 1.4630 & 1.6117 & 0.70 & 0.73 & 1.35 & 1.42 \\
		 & \multirow{ 1 }{*}{ KLR } & 0.4602 & 0.5440 & 1.0787 & 1.3580 & 0.57 & 0.60 & 2.68 & 1.91 \\
		 & \multirow{ 1 }{*}{ SLR } & 0.3956 & 0.4250 & 1.6651 & 1.7434 & 0.82 & 0.83 & 0.93 & 0.86 \\
		 & \multirow{ 1 }{*}{ LR } & 0.4653 & 0.5143 & 1.7161 & 1.5902 & 0.79 & 0.67 & 1.40 & 1.34 \\
		\cline{1-10}\multirow{ 1 }{*}{ Ada-sIDS } & \multirow{ 2 }{*}{ B-GP } & 0.5095 & 0.4883 & 1.3638 & 1.2129 & 0.49 & 0.52 & 1.62 & 1.83 \\
		\multirow{ 1 }{*}{ Ada-sRQS } & & 0.5586 & 0.5088 & 1.6736 & 1.7562 & 0.74 & 0.77 & 1.00 & 0.97 \\
		\hline \hline \cline{1-10}\multirow{ 2 }{*}{ Det-IDS } & \multirow{ 1 }{*}{ $\bar{p}$ } & 0.6848 & 0.4418 & 0.0211 & 0.1158 & 0.02 & 0.07 & 27.40 & 8.85 \\
		& \multirow{ 1 }{*}{ $\hat{p}_{\mathscr{L}_{0}}$ } & 0.6570 & 0.5756 & 0.0090 & 0.3639 & 0.01 & 0.18 & 24.10 & 6.99 \\
		\cline{1-10}\multirow{2}{*}{ RQS } & \multirow{ 1 }{*}{ $\bar{p}$ } & 0.7075 & 0.4846 & 0.0428 & 0.0649 & 0.02 & 0.07 & 27.88 & 12.36 \\
		& \multirow{ 1 }{*}{ $\hat{p}_{\mathscr{L}_{0}}$ } & 0.8442 & 0.4686 & 0.0477 & 0.0527 & 0.02 & 0.06 & 24.43 & 10.71 \\
		\hline
	\end{tabular}
}
\end{table}

\begin{table}[ht]
	\centering
	\caption[Performance metrics for the cubic test function $h_{3}$ implementing the spatial G-PBA.]{Performance Monte-Carlo metrics  for the test function $h_{3}$ at $T=20,000$.
		\label{tab:results-h3}
	}
\resizebox{\columnwidth}{!}{%
	\begin{tabular}{l|c|rr|rr|rr|rr}
		\hline
		\multirow{ 2 }{*}{Policy $\eta$} & \multirow{ 2 }{*}{$\hat{p}$} & \multicolumn{2}{c|}{$\hat{r}(f_{T})$  ($10^{-2}$) }& \multicolumn{2}{c|}{$\hat{l}_{0.95}(f_{T})$  ($10^{-2}$) }& \multicolumn{2}{c|}{$\hat{c}_{0.95}(f_{T})$   }& \multicolumn{2}{c}{$\hat{D}(f_{T};g_{T})$}\\
		& & $a_{0}{=}100$ & $a_{0}{=}250$  & $a_{0}{=}100$ & $a_{0}{=}250$ &  $a_{0}{=}100$ & $a_{0}{=}250$ & $a_{0}{=}100$ &  $a_{0}{=}250$  \\
		\cline{1-10}\multirow{ 4 }{*}{ sIDS } & \multirow{ 1 }{*}{ B-GP } & 4.3661 & 4.2959 & 8.2188 & 9.7472 & 0.57 & 0.64 & 1.87 & 1.57 \\
		 & \multirow{ 1 }{*}{ KLR } & 4.3403 & 4.5771 & 13.0221 & 12.0456 & 0.76 & 0.76 & 1.32 & 1.36 \\
		 & \multirow{ 1 }{*}{ SLR } & 4.4470 & 4.6160 & 7.3444 & 7.7896 & 0.49 & 0.46 & 2.35 & 2.23 \\
		 & \multirow{ 1 }{*}{ LR } & 3.7645 & 3.6936 & 10.6028 & 10.5738 & 0.71 & 0.70 & 1.52 & 1.39 \\
		\cline{1-10}\multirow{ 4 }{*}{ sRQS } & \multirow{ 1 }{*}{ B-GP } & 4.1913 & 4.0209 & 10.7298 & 10.8774 & 0.67 & 0.67 & 1.70 & 1.45 \\
		 & \multirow{ 1 }{*}{ KLR } & 3.9131 & 3.7121 & 14.2680 & 14.0897 & 0.81 & 0.84 & 1.17 & 0.98 \\
		 & \multirow{ 1 }{*}{ SLR } & 4.0451 & 4.1825 & 10.3663 & 10.2469 & 0.69 & 0.68 & 1.79 & 2.03 \\
		 & \multirow{ 1 }{*}{ LR } & 3.6513 & 4.1276 & 12.9502 & 11.5623 & 0.80 & 0.66 & 1.27 & 1.31 \\
		\cline{1-10}\multirow{ 1 }{*}{ Ada-sIDS } & \multirow{ 2 }{*}{ B-GP } & 4.1540 & 4.2334 & 11.1918 & 11.3152 & 0.68 & 0.67 & 1.11 & 1.11 \\
		\multirow{ 1 }{*}{ Ada-sRQS } &  & 4.1874 & 4.0915 & 11.6808 & 12.9052 & 0.67 & 0.76 & 1.39 & 1.05 \\
		\hline \hline \cline{1-10}\multirow{2}{*}{ Det-IDS } & \multirow{ 1 }{*}{ $\bar{p}$ } & 5.3257 & 4.8835 & 0.0187 & 0.4446 & 0.00 & 0.03 & 33.34 & 11.93 \\
		& \multirow{ 1 }{*}{ $\hat{p}_{\mathscr{L}_{0}}$ } & 5.7587 & 5.3403 & 0.0001 & 0.3862 & 0.00 & 0.01 & 27.94 & 10.01 \\
		\cline{1-10}\multirow{2}{*}{ RQS } & \multirow{ 1 }{*}{ $\bar{p}$ } & 5.1556 & 4.7262 & 0.0000 & 0.2978 & 0.00 & 0.01 & 37.77 & 15.16 \\
		& \multirow{ 1 }{*}{ $\hat{p}_{\mathscr{L}_{0}}$ } & 5.3406 & 4.7325 & 0.0001 & 0.7267 & 0.00 & 0.02 & 31.77 & 12.68 \\
		\hline
	\end{tabular}
}
\end{table}

\subsection{Evaluating the Quality of the Design}
\label{sub:qualityDesign}

To focus on the sampling aspect of spatial G-PBA, we examine more closely the designs $x_{1:n}^{(a,\eta)}$ obtained from implementing the sampling policy $\eta$ and batch size $a$. For this analysis we return to $h_1$ in ~\eqref{eq:g-ex1} and use B-GP as the representative surrogate with a fixed batch size of $a=100$. To judge the quality of $x_{1:n}^{(a,\eta)}$ for the SRFP, we compute the resulting exact posterior
$g_{n}^{(a,\eta)}( \cdot | x_{1:n}^{(a,\eta)})$ and evaluate the resulting
absolute residual $| \median (g_{n}^{a,\eta})- x^*|$ and corresponding length of $(1-\alpha)\%$-CI. A design that is better quantifying uncertainty about $X^*$ should have lower residuals and lower CI. We then benchmark the resulting metrics  against the following \textit{baseline} schemes which utilize the true $p(x)$ (and therefore the actual posterior density $g_{n}$):
	\begin{align} \tag{IDS}
	\label{eq:IDSTrueSampling}
	x_{n+1} &:= \argmax_{x \in (0,1)} \; \mathcal{I}(x,g_{n};p(x),a) \\
	\label{eq:RQSTrueSampling} \tag{RQS}
	x_{n+1} &:= G_{n}^{-1}(U_{n+1}),\quad  U_{n+1} \sim \mathsf{Unif}(0,1); \\
	\label{eq:UniformSampling} \tag{Unif}
	x_{n+1} &\sim \mathsf{Unif}(0,1).
	\end{align}

The sampling strategy \eqref{eq:IDSTrueSampling} is optimal in the sense of maximizing the expected KL distance between $g_{n}$ and $g_{n+a}$, and hence we use it as an upper bound on performance; \eqref{eq:UniformSampling} is a \textit{passive policy} used as a lower bound. To make the baseline policies comparable with the spatial G-PBA strategies, we implement the former with batched sampling using the transition function \eqref{eq:batched_updating_pba} and $a=100$. We also match the initialization step, employing $N_0 = T_0/a_0$ equidistant locations $x_{1:N_0}$ (with $T_0 = 5000$) to construct $g_{T_0}$, from which \eqref{eq:IDSTrueSampling}, \eqref{eq:RQSTrueSampling} and \eqref{eq:UniformSampling} are implemented.

Figure~\ref{fig:ResidualsActual} visualizes the results. We observe that sIDS is the sampling policy which best approximates the true IDS, and that all G-PBA strategies significantly outperform the~\eqref{eq:UniformSampling} baseline strategy.
   Interestingly,  both randomized and information-directed policies appear to have similar asymptotic performance in terms of average residuals and CI length.

\begin{figure}[htb]
	{
		\centering
		\includegraphics[width=0.9\textwidth]{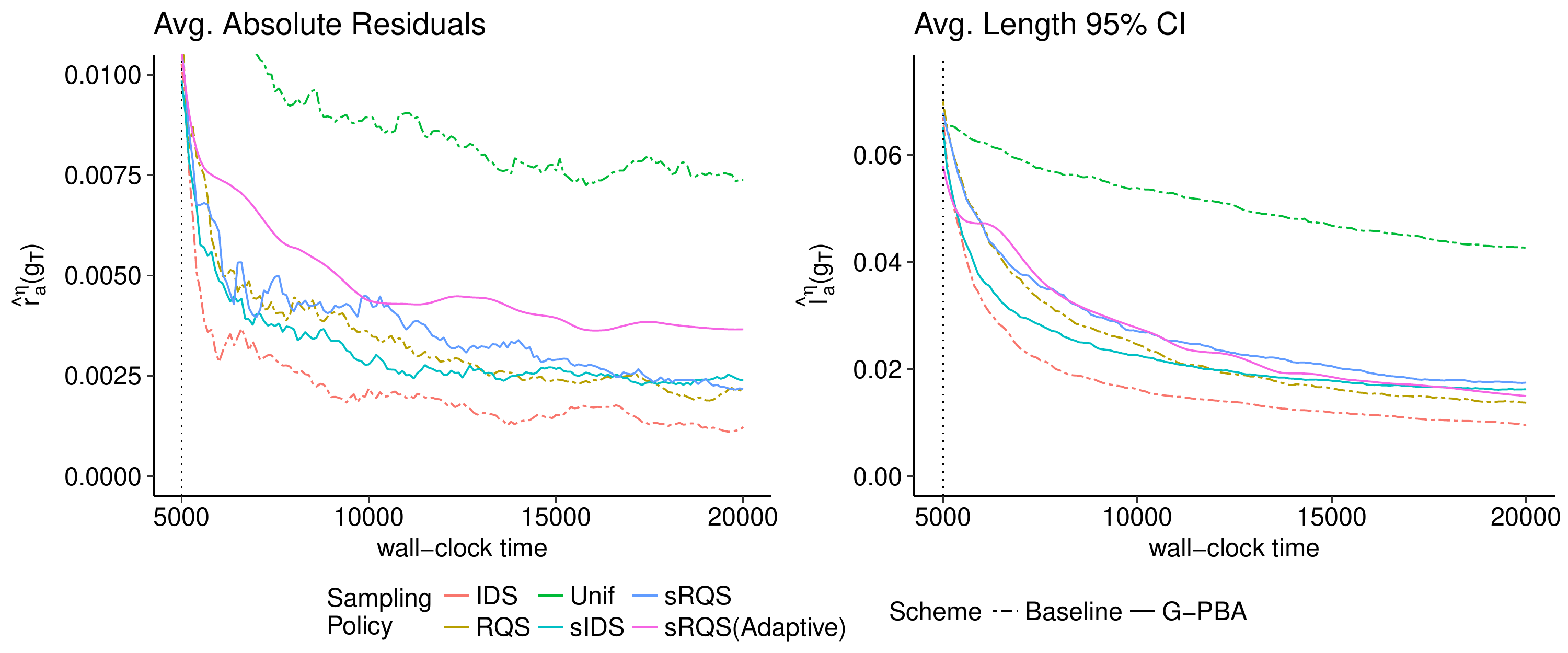}
		\caption[Comparison of spatial sampling policies with respect to baseline policies.]{Comparison of spatial sampling policies with respect to baseline policies using the true posteriors $g_T$. \emph{Left panel}: average absolute residuals, $\median(g_T),$ against wall-clock time $T$; \emph{Right panel:} average length of 95\% CI $l_{0.95}(g_{T})$ against $T$. Replication amounts fixed at $a_{n}=250\; \forall n$. \label{fig:ResidualsActual}}
	}
\end{figure}

\section{Case-Study: Root-Finding for Optimal Stopping}\label{sec:Results-AmericanOption}

In this section, we apply the spatial G-PBA Algorithm~\ref{alg:spatial-G-PBA} to solve the root-finding sub-routine for pricing a Bermudan Put option~\citep{ludkovski2015kriging}. Valuing a Bermudan option with maturity $\tilde{T}$ is equivalent to solving an optimal stopping problem
\[
V(t,x) := \sup_{\tilde{T} \geq \tau\geq t, \tau \in \mathcal{S}} \mathbb{E}\left[H(\tau, X_{\tau})|X_{t} = x \right],
\]
where $(X_t)$ is a stochastic process and $H(t,x)$ is the \emph{reward function}.  Assuming the classical discretized Black-Scholes model with time step $\Delta t$ we have that $(X_t)$ is a log-normal random walk and $H(t,x) := e^{-rt}(K^{Put} - x)_{+}$, where $K^{Put}$ is the strike price and  $r>0$ is the interest rate. In this setting, it is well-known that there is a unique \textit{exercise boundary} $x^*(t) \le K^{Put}$, and one should exercise at the first $t$ when $X_t$ drops below this boundary. Finding the exercise boundaries $\{x^*(t)\}$ reduces to solving a sequence of SRFPs, that is, pricing the Bermudan Put is equivalent to finding the solutions $x^*(t)$ of the equation $h(t,x)=0$ for $t=\tilde{T}-\Delta t,\tilde{T}-2\Delta t, \ldots, 0$, where $h(t,x):=  V(t,x) - H(t,x)$ is the \textit{timing value}.

 The \cite{longstaff2001valuing} method (LSM) recursively builds noisy simulators for $h(t,x)$ by generating forward paths $x_{t:\tilde{T}}$ of the state process $(X_t)$ and computing corresponding path-wise stopping times $\tau \equiv \tau(t+\Delta t,x_{t:\tilde{T}}) = \min\{ s > t: x_s \le x^*(s) \} \wedge \tilde{T}$. Namely, the pathwise difference $Z^{LSM}_{t}(x_t) := H({\tau},x_{\tau}) - H(t,x_{t})$ between future and immediate reward over the given trajectory $x_{t:\tilde{T}}$ satisfies $\mathbb{E}[Z_{t}(x_{t})] = h(x_t;t)$, matching the structure of the oracle \eqref{eq:z-oracle}. The random component $\epsilon(t,x)$ arises intrinsically from the randomness in the $X$-trajectory.

  We implement the spatial G-PBA for the Bermudan Put oracle using $K^{Put} = 40, r=0.06, \sigma = 0.25, \tilde{T}=1$ and $\Delta t = 0.04$, restricting the root-finding to the ``in-the-money'' domain  $x\in (25,40)$. Following the discussion in Section 5 in~\cite{rodriguez2017generalized}, to conform to the assumption of a symmetric noise distribution, we employ pre-averaging that considers the sign of an average of $R>1$ oracle evaluations:
\begin{equation}
\label{eq:PreAveragedOSOracle}
\bar{Y}^{LSM}_{R}(x):= \sign \bar{Z}_{R}(x), \quad \bar{Z}_{R}(x) := \frac{1}{ R}\sum_{r=1}^{R}Z^{LSM}_{r}(x).
\end{equation}
The role of pre-averaging is to alleviate statistical anomalies of $\epsilon(\cdot)$ via the Central Limit Theorem; the resulting oracle accuracy for this case-study is: $p_{R}^{LSM}(x):= \mathbb{P}(\bar{Y}^{LSM}_{R}(x) = \sign\{x^{*}-x\}).$ Below we continue to record the clock-time $T_n$ based on  underlying oracle evaluations (rather than the pre-averaged ones).

Due to the non-standard noise distribution and very low signal-to-noise ratio, this is a difficult root-finding problem; in particular since we keep the simulation budget to $T=20,000$.
\textcolor{black}{We implement the simulator~\eqref{eq:PreAveragedOSOracle} with $R = 25$,
which is roughly the minimal level of pre-averaging required to alleviate the skew of $Z^{LSM}$~\citep{rodriguez2017generalized}, and a batch size of $a_n\equiv 500$. Thus, the \textit{effective number of replicates} at each location is $\tilde{a} := a/R =20$. All surrogates are initialized with $\hat{\varphi}_{N_{0}}$  based on $T_{0}{=}0.25{\times}T{=}5000$ function evaluations using $N_{0} := T_{0}/a = 10$ design sites over the interval $(25,40)$.
The adaptive replication scheme~\eqref{eq:GP-AdaptiveBatching} is implemented with ${a}_{n+1}^{\nu} = \min\{\hat{a}_{n+1}^{\nu},a/R\},$ where $a_{0}^{\nu}:=1$ and $\nu_{n}^{Put}:= 0.5/n$, so that the maximum number of oracle evaluations is at most $a=500$ per querying location.}

\subsection{Results}

\begin{table}[ht]
	\centering
	\caption{Spatial G-PBA Performance in the Bermudan Put SRFP with simulation budget $T=20,000$. All metrics are averages across $MC=100$ macro-replications of the algorithms.
		\label{tab:mc-osp}
	}
	\begin{tabular}{c|c|c|c|c}
		\hline
		\multirow{ 1 }{*}{$\eta$} & \multirow{1}{*}{$\hat{p}$} & \multicolumn{1}{c|}{$\hat{r}(f_{T})$  }& \multicolumn{1}{c|}{$\hat{l}_{0.95}(f_{T})$ }& \multicolumn{1}{c}{$\hat{c}_{0.95}(f_{T})$ (in \%)}   \\
		\hline
		\multirow{ 4 }{*}{ sIDS } & \multirow{ 1 }{*}{ B-GP } & 0.3210 & 0.8903 & 69.00   \\
		& \multirow{ 1 }{*}{ KLR } & 0.3598 &  0.6351 &  53.33 \\
		& \multirow{ 1 }{*}{ SLR } & 0.3158 & 0.9878 & 77.00   \\
		& \multirow{ 1 }{*}{ LR } & 0.2753 & 1.0687 & 88.50   \\
		\cline{1-5}
		\multirow{ 4 }{*}{ sRQS } & \multirow{ 1 }{*}{ B-GP } & 0.2988 & 1.4064 & 86.00   \\
		& \multirow{ 1 }{*}{ KLR } & 0.3121 & 0.8209 & 62.00   \\
		& \multirow{ 1 }{*}{ SLR } & 0.3180 & 1.2005 & 74.50  \\
		& \multirow{ 1 }{*}{ LR } & 0.2913 & 1.4039 & 90.50   \\
		\cline{1-5}
		Ada-sIDS  & \multirow{ 2 }{*}{ B-GP }  & 0.2225 & 0.6944 & 80.00  \\
		Ada-sRQS  &  & 0.3011 & 1.1121 & 76.47  \\
		\hline
	\end{tabular}
\end{table}

 Table~\ref{tab:mc-osp} shows the average residuals, length of CI, and coverage probability of the spatial G-PBA schemes compared against the baseline root location $\hat{x}^{*}(t) \simeq 35.1249$ (here $t=0.6$) found in~\cite{rodriguez2017generalized}. This time, adaptive replication with the one-step sIDS policy~\eqref{eq:OneStepIDS} is the best-performing scheme. One reason could be that it allows for more sampling locations (median number of sampling locations was $\median(N_T) = 55$, as apposed to 40 for the fixed $a_n$ schemes). Among the rest, sIDS  policy coupled with the polynomial logistic regression model (LR) also performs very well, consistent with our findings in Section~\ref{sec:NumericExamples}. Relative to the non-spatial PBA in~\cite{rodriguez2017generalized} two important improvements are noted: (i) much better coverage probabilities, indicating the gains in learning $p(\cdot)$ and hence maintaining a reliable knowledge state; (ii) residuals below 0.25 while they used to be about 0.35.

\section{Conclusion}
\label{sec:Conclusions}

{We have developed a family of numerical schemes that extend generalized probabilistic bisection~\citep{rodriguez2017generalized} by modeling the unknown oracle accuracy $p(\cdot)$  through a spatial
surrogate based on non-parametric binomial regression. The spatial structure yields two key benefits:
(I) given the surrogate, the IDS  criterion $\cI$ can be predicted for any $x$, allowing direct optimization of next querying site selection like in standard PBA; (II) employing a GP surrogate quantifies the predictive uncertainty of additional samples and hence allows for adaptive batching schemes. Adaptive replication allows to automatically fine-tune exploration by reducing replication amounts in regions where $p(\cdot)$ is already learned well.  Our numeric experiments confirm the advantages of Spatial G-PBA relative to the original proposals in \cite{rodriguez2017generalized} with the new algorithm inducing more accurate root estimates and better quantifying the posterior uncertainty about $x^*$.}

Looking ahead, one motivation for considering PBA in the context of SRFP is its Bayesian flavor that allows in particular to apply informative priors $f_0$ as a way to warm-start the root search. This offers one way to lift PBA, which is intrinsically limited to a one-dimensional setting, to higher dimensions. The analogue of SRFP in two-dimensions is noisy (zero-)contour-finding, which can be viewed as a collection of root-finding problems in the first coordinate $x_1$, indexed by the second coordinate $x_2$. Assuming the zero-contour is smooth, one may then try to solve for a few $x^*(x_2)$ and then ``connect the dots'' through interpolation (or a further surrogate model). Such searches can be made efficient with G-PBA by using $f_N( \cdot ; x_2)$ as a basis for an informative prior $f_0(\cdot; x_2')$ at a new $x_2'$. We leave such investigations to future research.

\section*{Acknowledgments}
\textcolor{black}{Rodriguez is partially supported by the National Science and Technology Council of Mexico (CONACYT) and University of California Institute for Mexico and the United States (UCMEXUS) under grant CONACYT-216011. Ludkovski is partially supported by NSF DMS-1521743. We are also grateful to the UCSB Center for Scientific Computing from the CNSI and MRL: an NSF MRSEC (DMR-1720256).}

\appendix
\begin{appendices}

\section{Binomial GPs and Laplace Approximation}
\label{appendix:LaplaceApproximation}

\textbf{Binomial log-likelihood Gradient and Hessian.}
%
We use the Bernoulli link function  $\Theta(\varphi) = (1+e^{-\varphi})^{-1}$ which implies that conditional on $\varphi_{i}= \varphi(x_i)$, the number of positive responses $B_{i}:=\sum_{j=1}^{a_{i}}1_{\{Z_{j}>0\}}$ follows a binomial distribution $B_{i}{\sim}\mathsf{Bin}(a_{i},\Theta(\varphi_{i}))$ with log-likelihood function (in the latent $\varphi_{1:n}$):
\begin{align*}
l(\varphi_{1:n}) \equiv \log p(B_{1:n}|\varphi_{1:n},a_{1:n})
&= \sum_{j=1}^{n} \left\{\log {a_{i} \choose B_{i}} + B_{i}\log \theta(\varphi_{i}) + (a_{i} - B_{i})\log[1-\theta(\varphi_{i})]\right\}.
\end{align*}
Since $\Theta^{\prime}(\varphi):= \Theta(\varphi)[1-\Theta(\varphi)]$, the corresponding gradient vector
 $\bm{u}_{n}(\varphi_{1:n}):= \nabla l(\varphi_{1:n})$ is given by
\begin{align}
\partial_{\varphi_i} l( \varphi_{1:n}) = u_{i} &=B_{i} \frac{\Theta^{\prime}(\varphi_{i})}{\Theta(\varphi_{i})} - (a_{i}-B_{i})\frac{\Theta^{\prime}(\varphi_{i})}{1-\Theta(\varphi_{i})} =B_{i}[1-\Theta(\varphi_{i})] + (B_{i}-a_{i})\Theta(\varphi_{i}) \nonumber \\
\label{eq:GradientBinomialLikelihood}
&=B_{i} - a_{i}\theta(\varphi_{i}),\quad i=1,\ldots,n,
\end{align}	
which is a function of $\varphi_i$ only. 
Differentiating $\bm{u}_n$ again therefore yields the diagonal $n\times n$ Hessian matrix $\bm{W}_{n}(\varphi_{1:n})=-\Delta \log p(B_{1:n}|\varphi_{1:n},a_{1:n})$ as specified in \eqref{eq:HessianBinomialLikelihood}.


\textbf{Normal Approximation to the Joint Posterior Distribution}. 
By Bayes' rule the posterior $p(\varphi_{1:n}|\mathcal{D}_{n})$ is proportional to the Binomial likelihood $p(B_{1:n}|\varphi_{1:n},a_{1:n},x_{1:n})$ times the zero-mean GP prior $p(\varphi_{1:n}|x_{1:n})$. Taking the $\log$ of the unnormalized joint posterior we obtain
\begin{align}\notag
\mathscr{L}(\varphi_{1:n}) &\propto  \log p(B_{1:n}|\varphi_{1:n},a_{1:n},x_{1:n}) + \log p(\varphi_{1:n}|x_{1:n})\\
&:= \log p(B_{1:n}|\varphi_{1:n},a_{1:n},x_{1:n})  -\frac{1}{2}\varphi_{1:n}^{T}\bm{K}_{n}^{-1}\varphi_{1:n} - \frac{1}{2}\log|\bm{K}_{n}| - \frac{n}{2} \log 2 \pi.
\end{align}
Denote by $
\hat{\bm{\varphi}}_{n}:= \argmax_{\bm{\varphi}_{n}} \mathscr{L}(\bm{\varphi}_{n}) = \argmax_{\bm{\varphi}_{n}} p(\bm{\varphi}_{n}|\mathcal{D}_{n}).$ Expanding $\mathscr{L}(\cdot)$ around $\hat{\bm{\varphi}}_{n}$ gives
$\mathscr{L}(\bm{\varphi}_{n})= \mathscr{L}(\bm{\hat{\varphi}}_{n}) + \frac{1}{2}( \bm{\varphi}_{n} - \bm{\hat{\varphi}}_{n} )^{T}[\Delta \mathscr{L}(\bm{\hat{\varphi}}_{n})](\bm{\varphi}_{n} - \bm{\hat{\varphi}}_{n}) + \cdots$; where the linear term in the expansion is zero because the log-posterior density has zero derivative at its mode. As discussed in \cite{gelman2014bayesian}, the remainder terms of higher order fade in importance relative to the quadratic term when $\bm{\varphi}_{n}$ is close to $\bm{\hat{\varphi}}_{n}$ and the sample size  $n$  is large.
Taking first and second partial derivatives of $\mathscr{L}(\varphi_{1:n})$ with respect to $\varphi_{1:n}$ and combining with \eqref{eq:GradientBinomialLikelihood}-\eqref{eq:HessianBinomialLikelihood} we obtain:
\begin{align}
\nabla \mathscr{L}(\varphi_{1:n}) &= \bm{u}_{n}(\varphi_{1:n})- \bm{K}_{n}^{-1}\varphi_{1:n},\\
\Delta \mathscr{L}(\varphi_{1:n}) &= -\bm{W}_{n}(\varphi_{1:n}) - \bm{K}_{n}^{-1};
\end{align}
At the mode of $\mathscr{L}(\varphi_{1:n})$ we have
\begin{equation}
\label{eq:IRLS0}
\nabla \mathscr{L}(\bm{\hat{\varphi}}_{n}) = \bm{0} \quad \Rightarrow \qquad \bm{\hat{\varphi}}_{n} = \bm{K}_{n}\bm{u}_{n}(\bm{\hat{\varphi}}_{n})
\end{equation}
as a self-consistent nonlinear equation determining $\bm{\hat{\varphi}}_{n}$. In order to solve~\eqref{eq:IRLS0},
 an iterative procedure based on classical Newton-Raphson search is employed.

Next, the Hessian of the score function $\Delta \mathscr{L}( \varphi_n)$ is  interpreted as the inverse covariance matrix, leading to the Gaussian approximation
$q(\cdot|\mathcal{D}_{n})$ to the true posterior $p(\cdot|\mathcal{D}_{n})$
 \begin{equation}
 \label{eq:JointPosteriorLaplace}
 q(\cdot|\mathcal{D}_{n}) := \mathsf{N}(\cdot;\hat{\bm{\varphi}}_{n}, \bm{\Sigma}_{n}) \qquad \text{where} \quad \bm{\Sigma}_{n}\equiv (\bm{K}_{n}^{-1} + \bm{W}_{n}(\hat{\bm{\varphi}}_{n}))^{-1}.
 \end{equation}

%

\textbf{Predictive distribution.} The approximated predictive pdf $\varphi_*(x)$ at a test location $x\in (0,1)$ is Gaussian
$\varphi_*(x)\sim \mathsf{N}(m_{n}(x;\hat{\bm{\varphi}}_{n}),s_{n}^{2}(x;\hat{\bm{\varphi}}_{n}))$ with the mean $m_{n}(x;\hat{\bm{\varphi}}_{n})$ given by:
\begin{align} \notag
m_{n}(x;\hat{\bm{\varphi}}_{n}) &:= \int \mathbb{E}[\varphi(x)|\tilde{\varphi}_{1:n}] p(\tilde{\varphi}_{1:n}|\mathcal{D}_{n}) d\tilde{\varphi}_{1:n} \\ \notag
& = \bm{\kappa}_{n}^{T} \bm{K}_{n}^{-1} \int  \tilde{\varphi}_{1:n} p(\tilde{\varphi}_{1:n}|\mathcal{D}_{n}) d\tilde{\varphi}_{1:n} \\
& = \bm{\kappa}_{n}^{T} \bm{K}_{n}^{-1} \mathbb{E}[\varphi_{1:n}|\mathcal{D}_{n}] \simeq \bm{\kappa}_{n}^{T} \bm{K}_{n}^{-1} \hat{\bm{\varphi}}_{n};
\end{align}
where $\bm{\kappa}_{n}^{T}\equiv (\kappa(x_{1},x),\ldots,\kappa(x_{n},x))$, matching \eqref{eq:PosteriorPredictiveMean-n}. Likewise, the approximated predictive variance, $s_{n}(x;\hat{\bm{\varphi}}_{n})\equiv\mathbb{V}ar(\varphi(x)|\mathcal{D}_{n},\hat{\bm{\varphi}}_{n},x)$, is given by (cf.~\eqref{eq:PosteriorPredictiveVariance-n}):
\begin{align*}
s_{n}(x;\hat{\bm{\varphi}}_{n}) &:= \mathbb{E}[\mathbb{V}ar(\varphi(x)|\varphi_{1:n},x_{1:n})|\mathcal{D}_{n}] + \mathbb{V}ar(\mathbb{E}[\varphi(x)|\varphi_{1:n},x_{1:n},x]|\mathcal{D}_{n})\\
&= \mathbb{E}[\tau^{2}-\bm{\kappa}_{n}^{T}\bm{K}_{n}^{-1}\bm{\kappa}_{n} |\mathcal{D}_{n}] + \mathbb{V}ar(\bm{\kappa}_{n}^{T}\bm{\kappa}_{n}^{-1}\varphi_{1:n}|\mathcal{D}_{n}) \\
&= \tau^{2}-\bm{\kappa}_{n}^{T}\bm{K}_{n}^{-1}\bm{\kappa}_{n} + \bm{\kappa}_{n}^{T}\bm{K}_{n}^{-1}\mathbb{V}ar(\varphi_{1:n}|\mathcal{D}_{n})\bm{K}_{n}^{-1}\bm{\kappa}_{n}\\
&\simeq  \tau^{2}-\bm{\kappa}_{n}^{T}\bm{K}_{n}^{-1}\bm{\kappa}_{n} + \bm{\kappa}_{n}^{T}\bm{K}_{n}^{-1}(\bm{K}_{n}^{-1} + \bm{W}_{n}(\hat{\bm{\varphi}}_{n}))^{-1}\bm{K}_{n}^{-1}\bm{\kappa}_{n}\\
&=\tau^{2}- \bm{\kappa}_{n}^{T}(\bm{K}_{n} + \bm{W}_{n}(\hat{\bm{\varphi}}_{n})^{-1})^{-1}\bm{\kappa}_{n},
\end{align*}
where the last line is true via the matrix inversion lemma applied to $(\bm{K}_{n} + \bm{W}_{n}(\hat{\bm{\varphi}}_{n})^{-1})^{-1}$:
\begin{align*}
\bm{\kappa}_{n}^{T}(\bm{K}_{n} + \bm{W}_{n}(\hat{\bm{\varphi}}_{n})^{-1})^{-1}\bm{\kappa}_{n} &= \bm{\kappa}_{n}^{T}\{
\bm{K}_{n}^{-1} - \bm{K}_{n}^{-1}(\bm{K}_{n}^{-1} + \bm{W}_{n}(\hat{\bm{\varphi}}_{n}))^{-1}\bm{K}_{n}^{-1}
\}\bm{\kappa}_{n}\\
&=  \bm{\kappa}_{n}^{T}\bm{K}_{n}^{-1} - \bm{\kappa}_{n}^{T}(x)\bm{K}_{n}^{-1}(\bm{K}_{n}^{-1} + \bm{W}_{n}(\hat{\bm{\varphi}}_{n}))^{-1}\bm{K}_{n}^{-1}\bm{\kappa}_{n}.
\end{align*} 

\section{Predictive Variance Decomposition for Binomial GPs under Laplace Approximation (Theorem~\ref{thm:PosteriorPredictiveVariance})}
\begin{proof}
	\label{appendix:ProofPosteriorPredictiveVariance}
Set 
$\hat{\bm{\varphi}}_{n+1}\equiv (\hat{\varphi}_{1;n+1},\ldots,\hat{\varphi}_{n+1;n+1})
$
to be the $(n+1)$-dimensional estimated mode based on training data $\mathcal{D}_{n+1}$ obtained at locations $x_{1:n+1}$; and let
\[
\hat{\bm{W}}_{n+1;n+1}:= \mbox{diag}\{\hat{w}_{1;n+1},\ldots,\hat{w}_{n+1;n+1}\},\quad \hat{w}_{i}\equiv  w_{i}(\hat{\varphi}_{i;n+1})
\]
be the Hessian matrix~\eqref{eq:HessianBinomialLikelihood} evaluated at $\hat{\bm{\varphi}}_{n+1}$. Then, we have that the $(n+1) \times (n+1)$ covariance matrix $\bm{\Sigma}_{n+1}\equiv (\bm{K}_{n+1} + \hat{\bm{W}}_{n+1})^{-1})$ of the joint approximated posterior~\eqref{eq:JointPosteriorLaplace} can be partitioned as:
	\begin{equation}
	\label{eq:SigmaPartitioned}
	\bm{\Sigma}_{n+1} = \left(
	\begin{array}{cc}
	\bm{\Sigma}_{1:n;n+1} & \bm{\kappa}_{n}^{*}\\
	(\bm{\kappa}_{n}^{*})^{T} & \tau^{2} + \hat{w}_{n+1;n+1}^{-1}
	\end{array}
	\right),
	\end{equation}
	where $\bm{\kappa}_{n}^{*}:=(\kappa(x_{1},x_{n+1}),\ldots,\kappa(x_{n},x_{n+1}))^{T}$ is a $n \times 1$ column vector of covariances of $\varphi_{1:n}$ against $\varphi_{n+1}$, and $\tau^{2} = \kappa(x_{n+1},x_{n+1})$ is scalar.
	 Applying the Matrix Inversion Theorem~\cite{henderson1981deriving}, the inverse of~\eqref{eq:SigmaPartitioned} is:
	\begin{align*}
	\bm{\Sigma}_{n+1}^{-1} &= \left(
	\begin{array}{cc}
	\bm{\Sigma}_{1:n;n+1}^{-1} + (\bm{\Sigma}_{1:n;n+1}^{-1}\bm{\kappa}_{n}^{*})(\bm{\Sigma}_{1:n;n+1}^{-1}\bm{\kappa}_{n}^{*})^{T}a^{-1} & -\bm{\Sigma}_{1:n;n+1}^{-1}\bm{\kappa}_{n}^{*}a^{-1}\\
	-(\bm{\Sigma}_{1:n;n+1}^{-1}\bm{\kappa}_{n}^{*})^{T}a^{-1} & a^{-1}
	\end{array}
	\right) \\
	&= \left(
	\begin{array}{cc}
	\bm{\Sigma}_{n;n+1}^{-1}  & 0\\
	0 & 0
	\end{array}
	\right) + a^{-1}\left(
	\begin{array}{cc}
	(\bm{\Sigma}_{1:n;n+1}^{-1}\bm{\kappa}_{n}^{*})(\bm{\Sigma}_{1:n;n+1}^{-1}\bm{\kappa}_{n}^{*})^{T} & -\bm{\Sigma}_{1:n;n+1}^{-1}\bm{\kappa}_{n}^{*}\\
	-(\bm{\Sigma}_{1:n;n+1}^{-1}\bm{\kappa}_{n}^{*})^{T} & 1
	\end{array}
	\right),\\
	\text{where the scalar is}\quad a &:= (\tau^{2} + \hat{w}_{n+1;n+1}^{-1})-(\bm{\kappa}_{n}^{*})^{T}\bm{\Sigma}_{1:n;n+1}^{-1}\bm{\kappa}_{n}^{*} \\
	&=  \hat{w}_{n+1;n+1}^{-1} + (\tau^{2}-(\bm{\kappa}_{n}^{*})^{T}\bm{\Sigma}_{1:n;n+1}^{-1}\bm{\kappa}_{n}^{*}) 
	=  \hat{w}_{n+1;n+1}^{-1} + s_{n}^{2}(x_{n+1};\hat{\bm{\varphi}}_{1:n,n+1}).
	\end{align*}
	
	Substituting the expression for $\bm{\Sigma}_{n+1}^{-1}$ obtained above in the predictive variance formula \eqref{eq:PosteriorPredictiveVariance-n}, we have that the posterior predictive variance  given the dataset $\mathcal{D}_{n+1}$ is:
	\begin{align*}
	s_{n+1}^{2}(x_{n+1};\hat{\bm{\varphi}}_{n+1}) &:= \tau^{2} - (\bm{\kappa}_{n}^{*})^{T} \bm{\Sigma}_{n+1}^{-1}\bm{\kappa}_{n}^{*}\\ 
	&= \tau^{2} -
	\bm{u}^{T}
	\left\{
	\left(
	\begin{array}{cc}
	\bm{\Sigma}_{1:n;n+1}^{-1}  & 0\\
	0 & 0
	\end{array}
	\right) + a^{-1}\left(
	\begin{array}{cc}
	\bm{v}\bm{v}^{T} & -\bm{v}\\
	-\bm{v}^{T} & 1
	\end{array}
	\right)
	\right\}
	\bm{u} \\
	&= \tau^{2} - \bm{u}^{T}\left(
	\begin{array}{cc}
	\bm{\Sigma}_{n;n+1}^{-1}  & 0\\
	0 & 0
	\end{array}
	\right)\bm{u} -  a^{-1}\bm{u}^{T}
	\left(
	\begin{array}{cc}
	\bm{v}\bm{v}^{T} & -\bm{v}\\
	-\bm{v}^{T} & 1
	\end{array}
	\right)
	\bm{u}\\
	&= \left[\tau^{2} - (\bm{\kappa}_{n}^{*})^{T}\bm{\Sigma}_{1:n;n+1}^{-1}\bm{\kappa}_{n}^{*}\right] - a^{-1}[b^{2} - b\tau^{2} - \tau^{2}b + \tau^{4}], \qquad b_{1\times1}\equiv (\bm{\kappa}_{n}^{*})^{T}\bm{v}\\
	&= s_{n}^{2}(x_{n+1}; ;\hat{\bm{\varphi}}_{n}) - a^{-1}(\tau^{2} - b)^{2};
	\end{align*}
	where we set $\bm{v}_{n\times 1}\equiv \bm{\Sigma}_{1:n;n+1}^{-1}\bm{\kappa}_{n}^{*}$ and let $\bm{u}_{(n+1)\times 1}\equiv (\bm{\kappa}_{n}^{*}\  \tau^{2})^{T}$ be the concatenation of the vector $\bm{\kappa}_{n}^{*}$ and the scalar $\tau^{2}$. Simplifying, we finally get:
	\begin{align*}
	s_{n+1}^{2}(x_{n+1};\hat{\bm{\varphi}}_{n+1})&= s_{n}^{2}(x_{n+1};\hat{\bm{\varphi}}_{1:n;n+1}) - \frac{1}{a}(\tau^{2} - (\bm{\kappa}_{n}^{*})^{T}\bm{\Sigma}_{1:n;n+1}^{-1}\bm{\kappa}_{n}^{*})^{2}\\
	&= s_{n}^{2}(x_{n+1};\hat{\bm{\varphi}}_{1:n;n+1}) - \frac{(s_{n}^{2}(x_{n+1};\hat{\bm{\varphi}}_{1:n;n+1}))^{2} }{\hat{w}_{n+1}^{-1} + s_{n}^{2}(x_{n+1};\hat{\bm{\varphi}}_{n;n+1})}\\
	&=\frac{s_{n}^{2}(x_{n+1};\hat{\bm{\varphi}}_{1:n;n+1})\hat{w}_{n+1;n+1}^{-1}}{\hat{w}_{n+1;n+1}^{-1} + s_{n}^{2}(x_{n+1};\hat{\bm{\varphi}}_{1:n;n+1})} \\
	&=\left(\frac{1}{s_{n}^{2}(x_{n+1};\hat{\bm{\varphi}}_{1:n;n+1})} + \frac{1}{\hat{w}_{n+1;n+1}^{-1}}\right)^{-1}.
	\end{align*}
Finally, we notice that $\hat{w}_{n+1;n+1}:= a_{n+1}\Theta(\hat{\varphi}_{n+1;n+1})\left(1-\Theta(\hat{\varphi}_{n+1;n+1}) \right)$
which leads to \eqref{eq:PosteriorPredictiveVariance-nplus1}.
\end{proof}
\end{appendices} 

\bibliography{pbsm-June20}{}

	\end{document}